%% file: template.tex
\documentclass{article}

\usepackage{arxiv}

\usepackage[utf8]{inputenc} 
\usepackage[T1]{fontenc}    
\usepackage{hyperref}       
\usepackage{url}            
\usepackage{booktabs}       
\usepackage{amsfonts}       
\usepackage{nicefrac}       
\usepackage{microtype}      
\usepackage{amsthm}
\usepackage{amsmath,mathtools}
\usepackage{empheq}
\usepackage{amssymb}
\usepackage{macros}
\usepackage{subcaption}
\usepackage[square,sort,comma,numbers]{natbib}
\usepackage{algorithm}
\usepackage{algorithmic}
\usepackage{bm}

\usepackage[colorinlistoftodos,prependcaption]{todonotes}

\newcommand{\dsgd}{vanilla DecenSGD}
\newcommand{\algo}{\textsc{Matcha}}
\newcommand{\graph}{\mathcal{G}}
\newcommand{\eset}{\mathcal{E}}
\newcommand{\vset}{\mathcal{V}}

\newcommand{\nworkers}{m}
\newcommand{\sg}{g}
\newcommand{\Sg}{\mathbf{G}}
\newcommand{\tg}{\nabla F}
\newcommand{\Tg}{\nabla \mathbf{F}}
\newcommand{\lr}{\eta}
\newcommand{\lip}{L}
\newcommand{\obj}{F}
\newcommand{\vbnd}{\sigma^2}
\newcommand{\x}{\mathbf{x}} 
\newcommand{\avgx}{\overline{\mathbf{x}}}
\newcommand{\X}{\mathbf{X}}
\newcommand{\mixmat}{\mathbf{W}} 
\newcommand{\lap}{\mathbf{L}} 
\newcommand{\I}{\mathbf{I}} 
\newcommand{\J}{\mathbf{J}} 

\newcommand{\maxd}{\Delta}
\newcommand{\nmatching}{M}
\newcommand{\actrv}{\mathsf{B}}

\usepackage{cleveref}
\crefname{equation}{}{}
\Crefname{equation}{}{}
\crefname{thm}{theorem}{theorems}
\Crefname{thm}{Theorem}{Theorems}
\crefname{clm}{claim}{claims}
\Crefname{clm}{Claim}{Claims}
\Crefname{coro}{Corollary}{Corollaries}
\Crefname{lem}{Lemma}{Lemmas}
\Crefname{sec}{Section}{Sections}
\crefname{app}{appendix}{appendices}
\Crefname{app}{Appendix}{Appendices}
\Crefname{part}{Part}{Parts}
\crefname{prop}{proposition}{propositions}
\Crefname{prop}{Proposition}{Propositions}
\Crefname{propty}{Property}{Properties}
\crefname{figure}{fig.}{figures}
\Crefname{figure}{Figure}{Figures}
\crefname{defn}{definition}{definitions}
\Crefname{defn}{Definition}{Definitions}
\crefname{fact}{fact}{facts}
\Crefname{fact}{Fact}{Facts}
\crefname{appendix}{appendix}{appendices}
\Crefname{appendix}{Appendix}{Appendices}
\crefname{algo}{algorithm}{algorithms}
\Crefname{algo}{Algorithm}{Algorithms}
\crefname{algorithm}{algorithm}{algorithms}
\Crefname{algorithm}{Algorithm}{Algorithms}
\crefname{conj}{conjecture}{conjectures}
\Crefname{conj}{Conjecture}{Conjectures}
\crefname{obs}{observation}{observations}
\Crefname{obs}{Observation}{Observations}
\crefname{assump}{assumption}{assumptions}
\Crefname{assump}{Assumption}{Assumptions}
\crefname{rem}{remark}{remarks}
\Crefname{rem}{Remark}{Remarks}

\title{\textsc{Matcha}: Speeding Up Decentralized SGD via Matching Decomposition Sampling}

\author{
  Jianyu Wang\thanks{Correspondence to \texttt{jianyuw1@andrew.cmu.edu}} \\
  Carnegie Mellon University
   \And
 Anit Kumar Sahu \\
  Bosch Center for Artificial Intelligence
  \And
  Zhouyi Yang \\
  Carnegie Mellon University
  \And
  Gauri Joshi \\
  Carnegie Mellon University
  \And 
  Soummya Kar \\
  Carnegie Mellon University
}

\begin{document}
\maketitle

\begin{abstract}
This paper studies the problem of error-runtime trade-off, typically encountered in decentralized training based on stochastic gradient descent (SGD) using a given network. While a denser (sparser) network topology results in faster (slower) error convergence in terms of iterations, it incurs more (less) communication time/delay per iteration. 
In this paper, we propose {\algo}, an algorithm that can achieve a win-win in this error-runtime trade-off for any arbitrary network topology. The main idea of {\algo} is to parallelize inter-node communication by decomposing the topology into matchings. To preserve fast error convergence speed, it identifies and communicates more frequently over critical links, and saves communication time by using other links less frequently. Experiments on a suite of datasets and deep neural networks validate the theoretical analyses and demonstrate that {\algo} takes up to $5 \times$ less time than vanilla decentralized SGD to reach the same training loss.
\end{abstract}


\section{Introduction}
Due to the massive size of training datasets used in state-of-the-art machine learning systems, distributing the data and the computation over a network of worker nodes, i.e., data parallelism has attracted a lot of attention in recent years \citep{nedic2018network,mcmahan2016communication}. In this paper, we consider a decentralized setting without central coordinators (i.e., parameter servers) where nodes can only exchange parameters or gradients with their neighbors. This scenario is common and useful when performing training in sensor networks, multi-agent systems, as well as federated learning on edge devices.

\textbf{Error-Runtime Trade-off in Decentralized SGD.}
Previous works in distributed optimization have extensively studied the error convergence of decentralized SGD in terms of iterations or communication rounds \citep{nedic2009distributed,duchi2012dual,yuan2016convergence,zeng2016nonconvex,towfic2016excess,jakovetic2018convergence,scaman2018optimal}, mostly for (strongly)~convex loss functions. Recent works have extended the analyses to smooth and non-convex loss functions and subsequently applied it to distributed deep learning \citep{lian2017can,jiang2017collaborative,wang2018cooperative,assran2018stochastic}. \emph{However, all of the aforementioned works only focus on iterations complexity, i.e., the number of iterations required to achieve a target error. They do not explicitly consider or demonstrate how the topology affects the training runtime, that is, wall-clock time required to complete each iteration}. 
Densely-connected networks, when used appropriately, give faster error convergence. However, they incur a higher communication delay per iteration, which typically increases with the maximal node degree. Thus, in order to achieve the fastest error-versus-wallclock time convergence, it is imperative to jointly analyze the iteration complexity as well as the runtime per iteration. Only a few previous works explore the second factor from a theoretical perspective, see \citep{dutta2018slow, gupta2016model} --- these works consider the centralized parameter server model. 
\begin{figure}[!ht]
    \centering
    \begin{subfigure}{.18\textwidth}
    \centering
    \includegraphics[width=\textwidth]{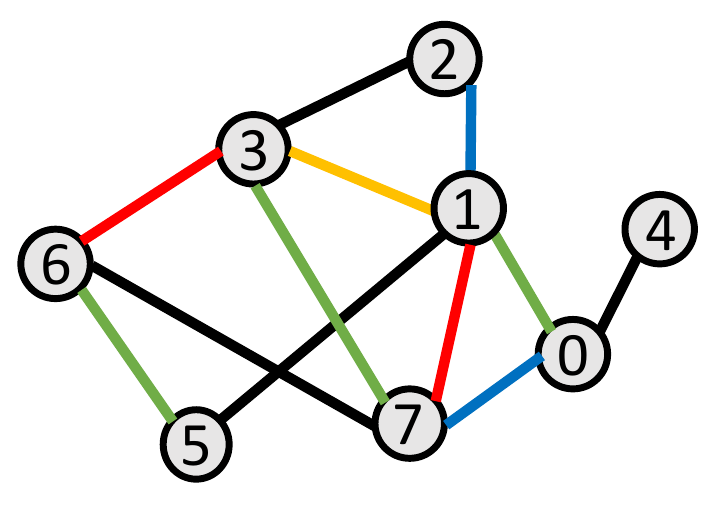}
    \caption{Base topology.}
    \label{fig:base_topo}
    \end{subfigure}%
    ~
    \begin{subfigure}{.25\textwidth}
    \centering
    \includegraphics[width=\textwidth]{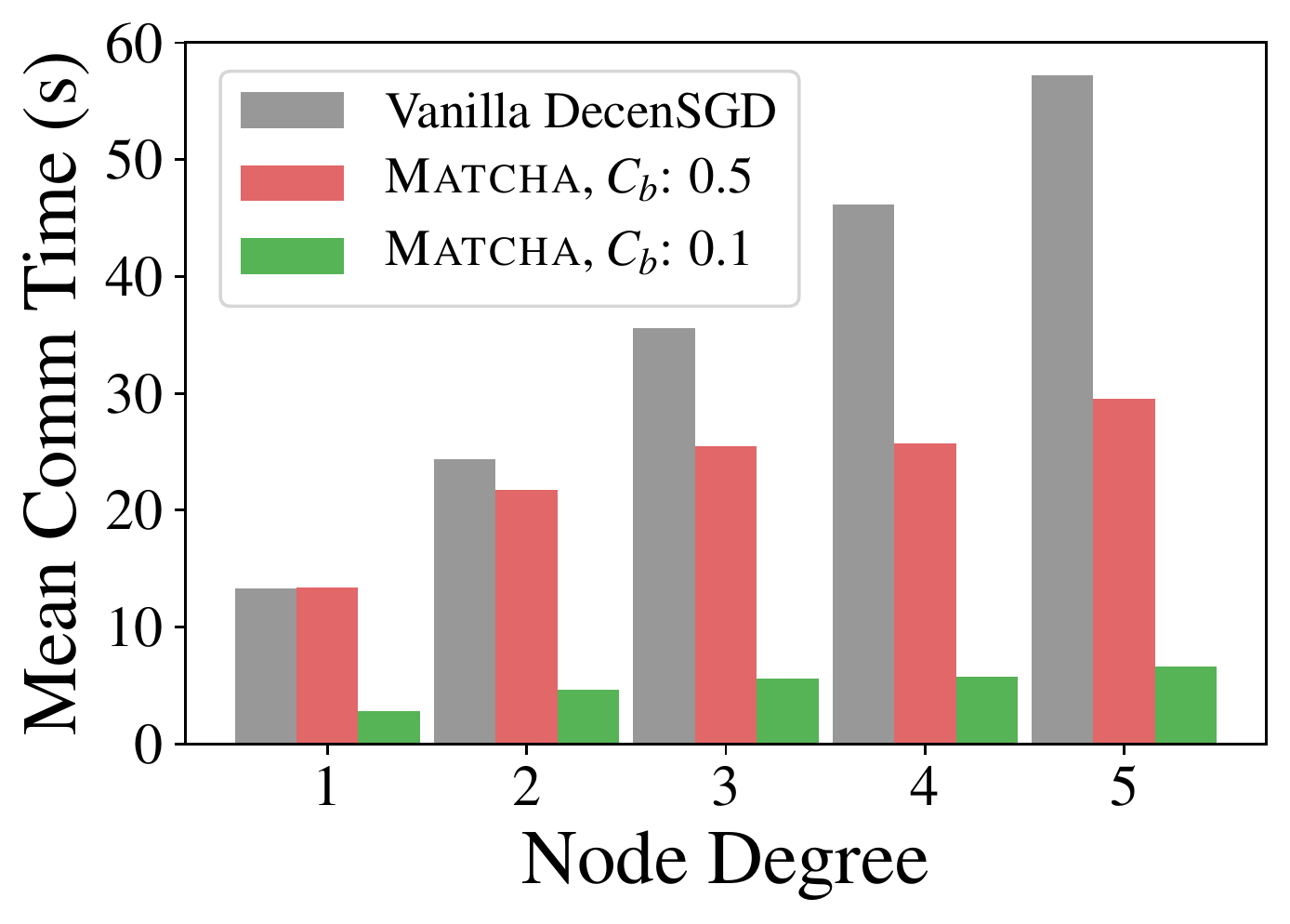}
    \caption{Comm. time per epoch.}
    \end{subfigure}%
    ~
    \begin{subfigure}{.25\textwidth}
    \centering
    \includegraphics[width=\textwidth]{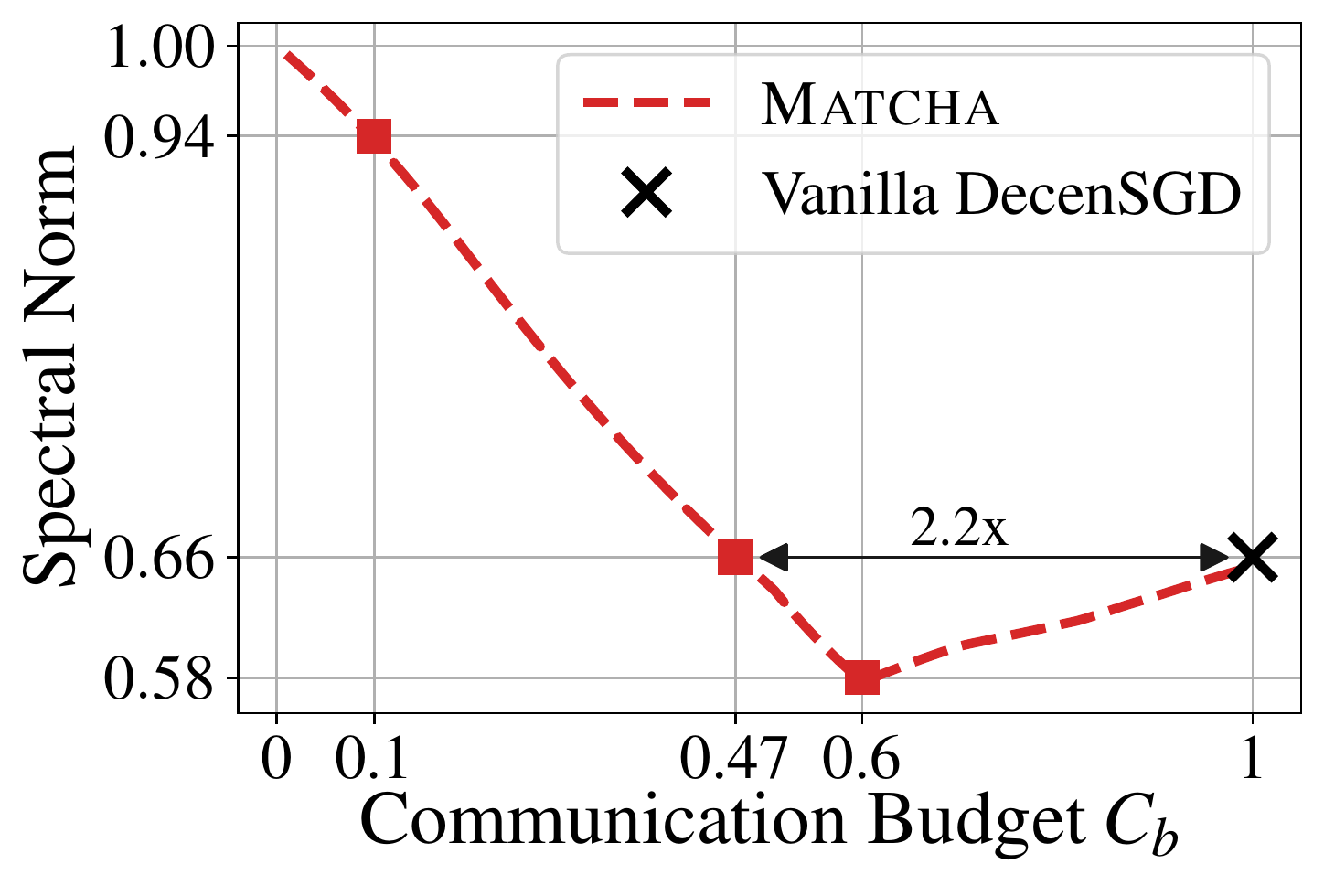}
    \caption{Spectral norm.}
    \label{fig:g1_rho}
    \end{subfigure}%
    ~
    \begin{subfigure}{.25\textwidth}
    \centering
    \includegraphics[width=\textwidth]{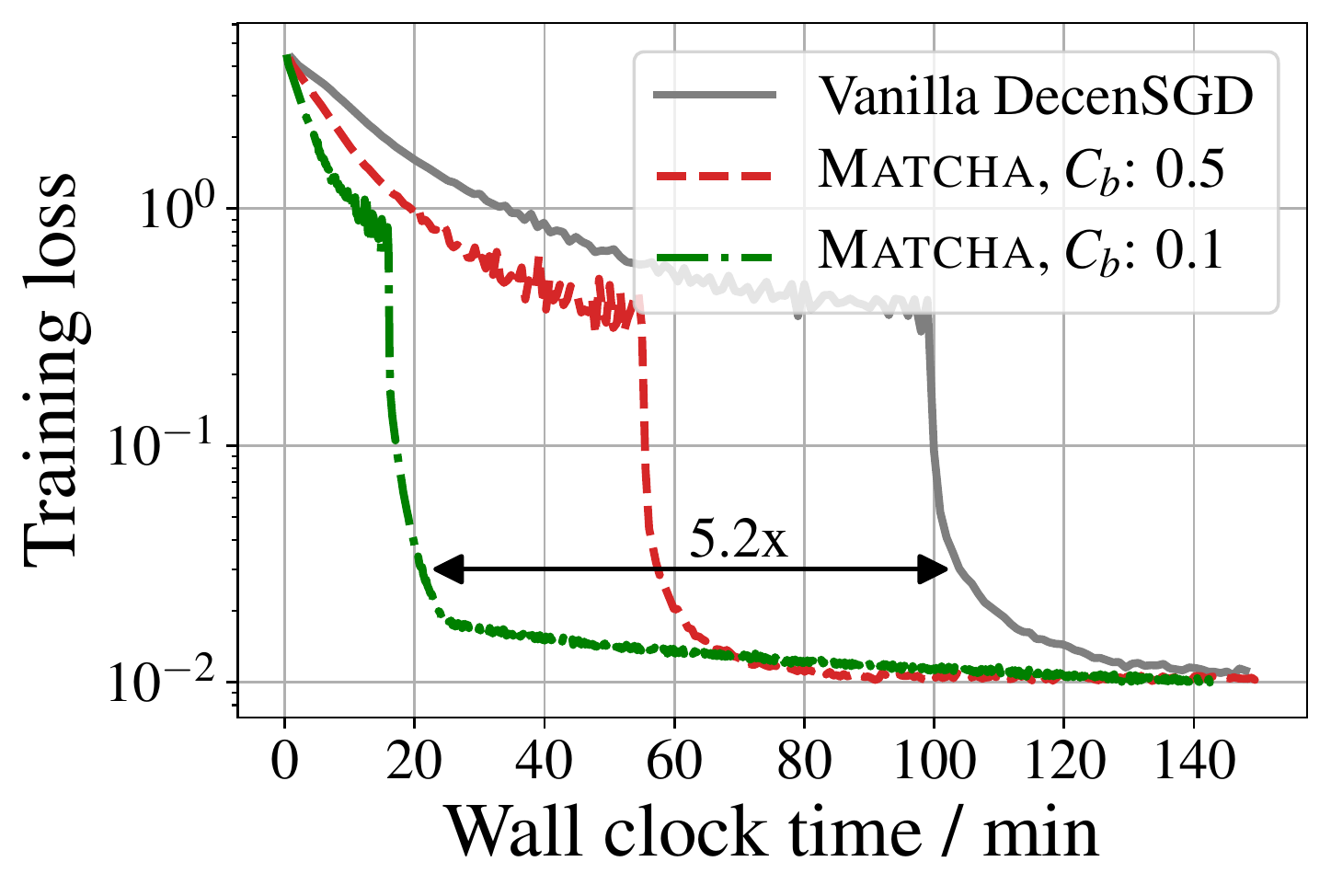}
    \caption{Training curves.}
    \end{subfigure}%
    ~
    \caption{Comparison of vanilla decentralized SGD (DecenSGD) and \textsc{Matcha}. (a) An example base topology generated using Erd\H{o}s-R\'enyi model. (b) {\algo} reduces the communication time non-uniformly. Nodes with higher degrees (for instance node $1$ with degree $5$) tend to have more redundant links. (c) Lower spectral norm yields better convergence rate. Communication budget $C_b$ represents the average frequency of communication over the links in the network. (d) Loss-versus-time curves when training WideResNet on CIFAR-100. } 
    \label{fig:comm_reduc}
    \vspace{-5pt}
\end{figure}

\textbf{Main Contributions.}
To the best of our knowledge, this is the first work that attempts to strike the best error-runtime trade-off in decentralized SGD by carefully \emph{tuning the frequency of inter-node communication}. We propose {\algo}, a decentralized SGD method based on matching decomposition sampling, and demonstrate the effectiveness of {\algo} both theoretically and empirically. 
The key ideas and main contributions of this paper are as follows.

\begin{enumerate}

\item \textbf{Saving Communication Time by Using Disjoint Links.} The communication delay of the model synchronization step in decentralized SGD is typically proportional to the maximal node degree \cite{tsianos2012communication}. To reduce this communication delay without hurting convergence, we propose a simple, yet powerful idea of decomposing the graph into matchings. Each matching is a set of disjoint links that communicate in parallel, as illustrated by the colored links in \Cref{fig:comm_reduc}(a). The probability of activating each matching is optimized so as to maximize the algebraic connectivity of the expected topology (captured by the second smallest eigenvalue $\lambda_2$ of the graph Laplacian matrix). This results in more frequent communication over connectivity-critical links (ensuring fast error-versus-iterations convergence) and less frequent over other links (saving inter-node communication time). 


\item \textbf{Flexible Communication Budget.} {\algo} allows the system designer to set a flexible communication budget $C_b$, which represents the average frequency of communication over the links in the network. When $C_b = 1$, {\algo} reduces to vanilla decentralized SGD studied in \cite{lian2017can}. When we set $C_b < 1$, {\algo} carefully reduces the communication frequency of each link, depending upon its importance in maintaining the overall connectivity of the graph. For example, observe in \Cref{fig:comm_reduc}(b) that by setting $C_b=0.1$, {\algo} achieves a $1/0.1 = 10 \times$ reduction in expected communication time per iteration. Note that the communication reduction is much larger for higher-degree nodes than for low-degree nodes, as shown in \Cref{fig:comm_reduc}(b). This judicious asymmetry in the communication reduction helps {\algo} to preserve fast error-versus-iterations convergence.  


\item \textbf{Same or Faster Error Convergence than Vanilla Decentralized SGD.} In \Cref{sec:analysis} we present a convergence analysis of {\algo} for non-convex objectives and illustrate the dependence of the error on $\rho$, the spectral norm of the mixing matrix (defined formally later). 
This analysis shows that for a suitable communication budget, {\algo} achieves the same or smaller $\rho$ as vanilla decentralized SGD --- a smaller $\rho$ implies faster error convergence. For example, observe in \Cref{fig:comm_reduc}(c) that {\algo} has the same spectral norm as vanilla decentralized SGD (DecenSGD) with a $2.2 \times$ less communication budget per iteration, and if we set $C_b = 0.6$ then the spectral norm is even lower. \emph{In this case, contrary to intuition, {\algo} not only reduces the per-iteration communication delay but also gives faster error-versus-iterations convergence.} 

\item \textbf{Experimental Results on Error-versus-wallclock Time Convergence.} In \Cref{sec:exps} we evaluate the performance of {\algo} on a suite of deep learning tasks, including image classification on CIFAR-10/100 and language modeling on Penn Treebank, and for several base graph topologies including Erd\H{o}s-R\'enyi and geometric graphs. The empirical results consistently corroborate the theoretical analyses and show that {\algo} can get up to $5.2 \times$ reduction in wall-clock time (computation plus communication time) to achieve the same training accuracy as vanilla decentralized SGD, as illustrated in \Cref{fig:comm_reduc}(d). Moreover, {\algo} achieves test accuracy that is comparable or better than vanilla decentralized SGD. 

\item \textbf{Extendable to Other Subgraphs and Computations.} While we currently decompose the topology into matchings, our approach can be extended to other sub-graphs such as edges or cliques. Furthermore, going beyond decentralized SGD, the core idea of {\algo} is extendable to any distributed computation or consensus algorithm that requires frequent synchronization between neighboring nodes.
\end{enumerate}

%


\textbf{Connection to Gradient Compression and Quantization Techniques.} By reducing the frequency of inter-node communication, {\algo} effectively performs more local SGD updates at each node between model synchronization steps, i.e., consensus steps. Thus, {\algo} belongs to the class of local-update SGD methods recently studied in \cite{stich2018local,wang2018cooperative, Wang2018Adaptive}. An orthogonal way of reducing inter-node communication is to compress or quantize inter-node model updates \citep{tang2018communication,koloskova2019decentralized}. These gradient compression techniques reduce the amount of data that is transmitted per round, whereas local update methods such as {\algo} reduce the frequency of communication. Besides, compression techniques involve encoding and decoding data at each iteration, which incurs additional overhead. {\algo} does not incur such additional overhead at runtime -- the sequence of subgraphs is pre-determined before the training starts. {\algo} can be easily combined with these complementary compression techniques to give a further reduction in the overall communication time. 


\section{Problem Formulation and Preliminaries}
Consider a network of $\nworkers$ worker nodes. The communication links connecting the nodes are represented by an arbitrary possibly sparse undirected connected graph $\graph$ with vertex set $\vset = \{1, 2, \dots, \nworkers\}$ and edge set $\eset \subseteq \vset \times \vset$. Each node $i$ can only communicate (i.e., exchange model parameters or gradients) with its neighbors, that is, it can communicate with node $j$ only if $(i,j)\in \eset$. 

Each worker node $i$ only has access to its own local data distribution $\mathcal{D}_i$. Our objective is to use this network of $m$ nodes to train a model using the joint dataset. In other words, we seek to minimize the objective function $F(\x)$, which is defined as follows:
\begin{align}\label{eqn:obj}
\begin{split}
    &F(\x) \triangleq \frac{1}{\nworkers}\sum_{i=1}^\nworkers \obj_i(\x) = \frac{1}{\nworkers}\sum_{i=1}^\nworkers \Exs_{\mathbf{s} \sim \mathcal{D}_i}\brackets{\ell(\x;\mathbf{s})}
\end{split}
\end{align}
where $\x$ denotes the model parameters (for instance, the weights and biases of a neural network), $\obj_i(\x)$ is the local objective function, $\mathbf{s}$ denotes a single data sample, and $\ell(\x;\mathbf{s})$ is the loss function for sample $\mathbf{s}$, defined by the learning model.


\textbf{Decentralized SGD (DecenSGD).} 
Decentralized SGD can be tracked back to the seminal work of \citep{tsitsiklis1986distributed}. It is a natural yet effective way to optimize the empirical risk \Cref{eqn:obj} in the considered decentralized setting. The algorithm simultaneously incorporates the obtained information from the neighborhood~(consensus) and the local gradient information as follows\footnote{One can also use another update rule: $\x_{i}^{(k+1)} = \sum_{j=1}^{\nworkers} W_{ji}\x_j^{(k)} - \sg(\x_i^{(k)};\xi_i^{(k)})$. All insights and conclusions in this paper will remain the same.}:
\begin{align}\label{eqn:udpate_rule}
    \x_{i}^{(k+1)} = \underbrace{\sum_{j=1}^{\nworkers} W_{ij}}_{\text{consensus step}}\underbrace{\brackets{\x_j^{(k)} - \sg(\x_j^{(k)};\xi_j^{(k)})}}_{\text{local gradient step}}
\end{align}
where $\xi_j^{(k)}$ denotes a mini-batch sampled 
from local data distribution $\mathcal{D}_j$ at iteration $k$, $\sg(\x;\xi)$ denotes the stochastic gradient, and $W_{ij}$ is the $(i,j)$-th element of the mixing matrix $\mixmat \in \mathbb{R}^{\nworkers \times \nworkers}$. In particular, $W_{ij}\neq 0$ only if node $i$ and node $j$ are connected, i.e., $(i,j) \in \eset$. Setting the mixing matrix $\mixmat$ to be symmetric and doubly stochastic is one way to ensure that the nodes reach consensus in terms of converging to the same stationary point. For instance, if node $1$ is only connected to nodes $2$ and $3$, then the first row of $\mixmat$ can be $[1-2\alpha,\alpha,\alpha,0,\dots,0]$, where $\alpha$ is constant.

\textbf{Communication Time Model.} 
In decentralized SGD methods, each node communicates with all of its neighbors, and the delay in performing this local synchronization typically increases with the node degree. Thus, the node with the highest degree in the graph becomes the bottleneck and the communication time per iteration can be modeled as:
\begin{align}
    \text{Comm. Time per Iteration} = t(\maxd(\graph)),
\end{align}
where $\maxd$ is the maximal degree of the graph and $t(\cdot)$ is a monotonically increasing function. For brevity of presentation, we will focus on a linear-scaling rule (i.e., $t(\Delta)= \Delta$) in this paper, as considered in previous works \citep{tsianos2012communication,chow2016expander, lian2017asynchronous} and also observed in experimental result \Cref{fig:comm_reduc}(b). But {\algo} can be easily extended to other delay-scaling rule $t(\cdot)$. Without loss of generality, we can assume the communication (sending and receiving model parameters) over one link costs $1$ unit of time. Thus, the communication per iteration takes at least $t(\maxd)$ units of time.




\textbf{Preliminaries from Graph Theory.}
The communication graph $\graph(\vset,\eset)$ can be abstracted as an adjacency matrix $\mathbf{A} \in \mathbb{R}^{\nworkers \times \nworkers}$. In particular, $A_{ij}=1$ if $(i,j)\in \eset$; $A_{ij}=0$ otherwise. The graph Laplacian $\lap$ is defined as $\lap = \text{diag}(d_1,\dots,d_\nworkers) - \mathbf{A}$, where $d_i$ denotes the $i$-th node's degree. When $\graph$ is a connected graph, the second smallest eigenvalue $\lambda_2$ of the graph Laplacian is strictly greater than $0$ and referred to as \emph{algebraic connectivity} in \cite{bollobas2013modern}. A larger value of $\lambda_2$ implies a denser graph. Moreover, we will use the notion of \emph{matching},  defined as follows: a matching in $\graph$ is a subgraph of $\graph$, in which each vertex is incident with at most one edge.

\section{{\algo}: Proposed Matching Decomposition Sampling Strategy}
Following the intuition that it is beneficial to \emph{communicate over critical links more frequently and less over other links}, the algorithm consists of three key steps as follows, all of which can be done before training starts. A brief illustration is also shown in \Cref{fig:illustration}.
\begin{figure}[!htbp]
    \centering
    \includegraphics[width=.75\textwidth]{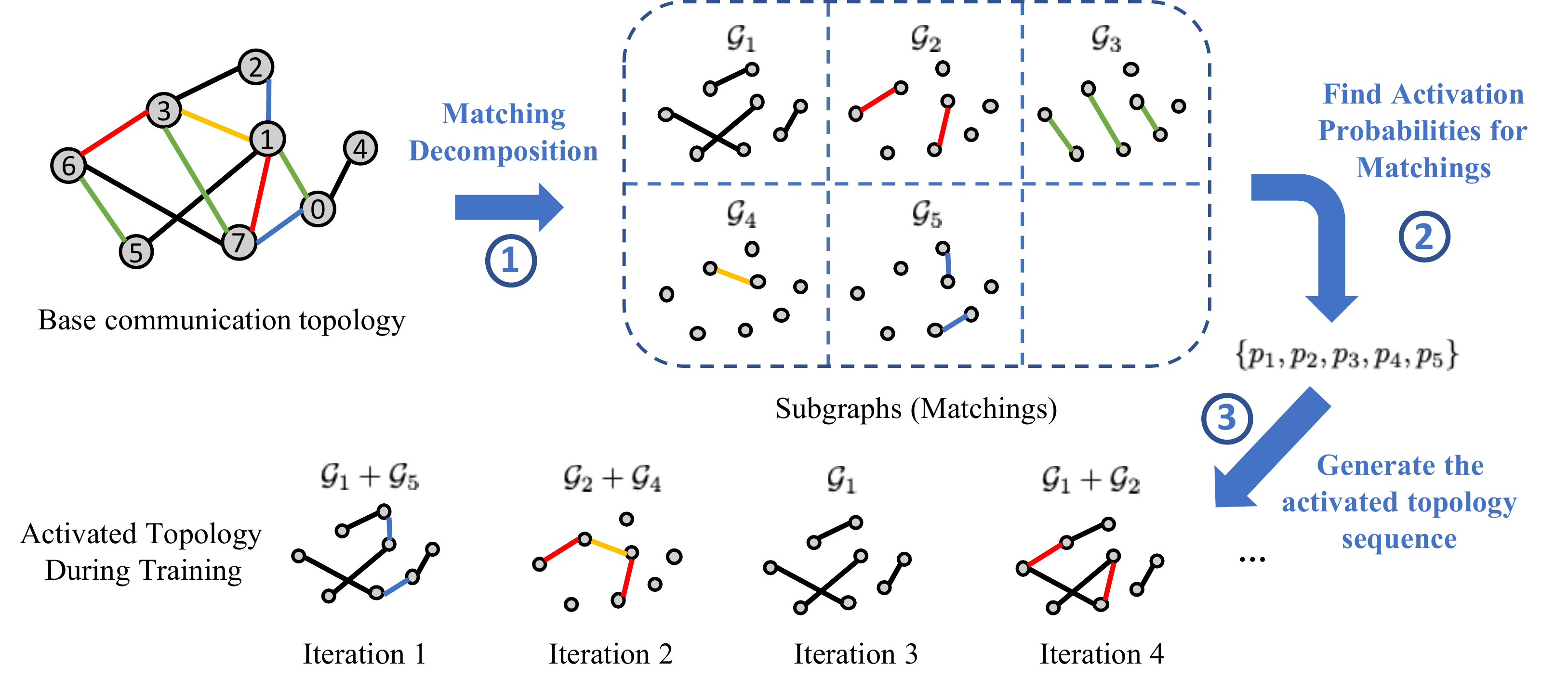}
    \caption{Illustration of {\algo}. Given the base communication graph, we decompose it into disjoint subgraphs (in particular, matchings, in order to allow parallel communications). Then, at each communication round, we carefully sample a subset of these matchings to construct a sparse subgraph of the base topology. Worker nodes are synchronized only through the activated topology.}
    \label{fig:illustration}
\end{figure}

\textbf{Step 1: Matching Decomposition.}
First, we decompose the base communication graph into total $\nmatching$ disjoint matchings, i.e., $\graph(\vset,\eset) = \bigcup_{j=1}^\nmatching \graph_j(\vset,\eset_j)$ and $\eset_i\bigcap\eset_j = \varnothing,\forall i \neq j$, as illustrated in \Cref{fig:illustration}. We can use any matching decomposition algorithm to find the $\nmatching$ disjoint matchings. 
For example, the polynomial-time edge coloring algorithm in \cite{misra1992constructive}, provably guarantees that the number of matchings $\nmatching$ equals to either $\maxd(\graph)$ or $\maxd(\graph)+1$, where $\maxd(\graph)$ is the maximal node degree.

For the communication time model where the delay is a monotonically increasing function of the node degree, using matchings is a natural solution because they minimize the effective node degree while maximizing the number of parallel information exchanges between nodes. Since all the nodes in a matching have degree one, the inter-node links are disjoint and these links can operate in parallel. Thus, according to our communication delay model, the communication time for one matching is exactly $t(1) = 1$ unit. 

Observe that communicating sequentially over all the matchings is also a simple and efficient way to implement the consensus step in decentralized training. 
Thus, the communication time per iteration will be equal to the number of matchings $\nmatching$.
%

\textbf{Step 2: Computing Matching Activation Probabilities.}
In order to control the total communication time per iteration, we assign an independent Bernoulli random variable $\actrv_j$, which is $1$ with probability $p_j$ and $0$ otherwise, to each matching $\forall j \in \{1,\dots,\nmatching\}$. The links in matching $j$ will be used for information exchange between the corresponding worker nodes only when the realization of $\actrv_j$ is $1$. As a result, the communication time per iteration can be written as
\begin{align}
    \textrm{Expected Comm. Time} = \Exs\brackets{\sum_{j=1}^\nmatching \actrv_j} = \sum_{j=1}^\nmatching p_j. \label{eqn:mct}
\end{align}
where $p_j$ is the \emph{activation probability} of the $j^{th}$ matching. When all $p_j$'s equal to $1$, all the $M$ matchings are activated in every iteration. Thus, when $p_j = 1$ for all $j \in \{1, \dots, \nmatching\}$, the algorithm reduces to {\dsgd} and takes $\nmatching$ units of time to finish one consensus step.
%
%
%
To reduce the expected communication time per iteration given by \eqref{eqn:mct}, we define \emph{communication budget} $C_b>0$, and impose the constraint $\sum_{j=1}^{\nmatching} p_j \leq C_b M$. 
For example, $C_b=0.1$ means that the expected communication time per iteration of {\algo} be $10\%$ of the time per iteration of {\dsgd}). 

As mentioned before, the key idea of {\algo} is to give more importance to critical links. This is achieved by choosing a set of activation probabilities that maximize the connectivity of the expected graph given a communication time constraint. That is, we solve the optimization problem: 
\begin{align}\label{eqn:opt_p}
\begin{split}
    \max_{p_1,\dots,p_\nmatching} \quad & \lambda_2\parenth{\textstyle{\sum_{j=1}^{\nmatching}} p_j \lap_j} \\
    \text{subject to} \quad & \textstyle{\sum_{j=1}^\nmatching} p_j \leq C_b \cdot \nmatching, \\
    & \ 0 \leq p_j \leq 1, \ \forall j \in \{1,2,\dots,\nmatching\},
\end{split}
\end{align}
where $\lap_j$ denotes the Laplacian matrix of the $j$-th subgraph and $\sum_{j=1}^{\nmatching} p_j \lap_j$ can be considered as the Laplacian of the expected graph. Moreover, recall that $\lambda_2$ represents the algebraic connectivity and is a concave function \citep{kar2008sensor,bollobas2013modern}. Thus, it follows that \eqref{eqn:opt_p} is a convex problem and can be solved efficiently. Typically, a larger value of $\lambda_2$ implies a better-connected graph.

\textbf{Step 3: Generating Random Topology Sequence.}
Given the activation probabilities obtained by solving \eqref{eqn:opt_p}, in each iteration, we generate an independent Bernoulli random variable $B_j$ for each matching. Thus, in the $k$-th iteration, the activated topology $\graph^{(k)} = \bigcup_{j=1}^\nmatching \actrv_j^{(k)} \graph_j$, which is sparse or even disconnected. Consequently, one can also obtain the corresponding Laplacian matrix sequence: $\lap^{(k)} = \sum_{j=1}^\nmatching \actrv_j^{(k)}\lap_j, \forall k\in\{1,2,\dots\}$. \emph{All of these information can be obtained and assigned apriori to worker nodes before starting the training procedure}.

\textbf{Extension to Other Design Choices.}
Our proposed {\algo} framework of activating different subgraphs in each iteration is very general -- it can be extended to various other delay models, graph decomposition methods and algorithms involving decentralized averaging. For example, instead of using a linear-delay scaling model, one can assume the communication time is a general increasing function $t(\Delta(\graph))$. In this case, we only need to change the constraint in \Cref{eqn:opt_a} to $t(\sum_{j=1}^\nmatching p_j) \leq C_b\cdot t(M)$. Instead of activating all matchings independently, one can choose to activate only one matching at each iteration. Instead of assuming all links cost same amount of time, one can model the communication time for each link as a random variable and modify the formula \Cref{eqn:mct} accordingly. Finally, instead of decomposing the base topology into matchings, each subgraph can be a single edge or a clique in the base graph. 

\textbf{Why not use Periodic Decentralized SGD?} A naive way to reduce the communication time is to perform infrequent synchronization, also called \emph{Periodic DecenSGD (P-DecenSGD)}, and studied in previous works \citep{tsianos2012communication,wang2018cooperative}. In P-DecenSGD, each node makes several local model updates before synchronizing with its neighbors, that is, all links in the base topology are activated together ($\actrv_1=\cdots=\actrv_\nmatching=1$) after every few iterations. In this case, $C_b$ equals to the communication frequency of the whole base graph. A drawback of this method is that it treats all links equally. By varying the communication frequency across links depending on how they affect the algebraic connectivity of the expected topology, {\algo} can give the same communication reduction as P-DecenSGD but with better error convergence guarantees. In \Cref{sec:analysis,sec:exps}, we will use P-DecenSGD as a benchmark for comparison.


\section{Further Optimizing {\algo} For Decentralized SGD}
In order to make the best use of the sequence of activated subgraphs generated by {\algo} when running decentralized SGD with a time-varying topology, we need to optimize the proportions in which the local models are averaged together in the consensus step of Eqn.~\Cref{eqn:udpate_rule}. A common practice is to use an equal weight mixing matrix as \citep{xiao2004fast,kar2008sensor,duchi2012dual}:
\begin{align}
    \mixmat^{(k)} = \I - \alpha \lap^{(k)} = \I - \alpha \sum_{j=1}^\nmatching \actrv_j^{(k)}\lap_j, \label{eqn:basic_w}
\end{align}
where $\lap^{(k)}=\sum_{j=1}^\nmatching \actrv_j^{(k)}\lap_j$ denotes the graph Laplacian at the $k$-th iteration. Each matrix $\mixmat^{(k)}$ is symmetric and doubly stochastic by construction, and the parameter $\alpha$ represents the weight of neighbor's information in the consensus step. 
To guarantee the convergence of decentralized SGD to a stationary point of smooth and non-convex losses (formally stated in \Cref{thm:basic}), the \emph{spectral norm} $\rho = \opnorm{\Exs[\mixmat^{(k)\top}\mixmat^{(k)}]-\frac{1}{\nworkers}\one\one\tp}$ must be less than $1$. 
In general, a smaller value of $\rho$ leads to a smaller optimization error bound. 

In the following theorem, we show that for {\algo} with arbitrary communication budget $C_b > 0$, one can always find a value of $\alpha$ for which the resulting spectral norm $\rho < 1$, which in turn guarantees the convergence of {\algo} to a stationary point.

\begin{thm}\label{thm:rho}
Let $\{\lap^{(k)}\}$ denote the sequence of Laplacian matrices generated by \textsc{Matcha} with arbitrary communication budget $C_b > 0$ for a connected base graph $\graph$. If the mixing matrix is defined as $\mixmat^{(k)} = \I - \alpha \lap^{(k)}$, then there exists a range of $\alpha$ such that $\rho = \opnorm{\Exs[\mixmat^{(k)\top}\mixmat^{(k)}]-\frac{1}{\nworkers}\one\one\tp} < 1$, which guarantees the convergence of {\algo} to a stationary point. 

The value of $\alpha$ can be obtained by solving the following semi-definite programming problem:
\begin{align}
\label{eqn:opt_a}
\begin{split}
    \min_{\rho,\alpha,\beta} \quad& \rho, \\
    \text{subject to} \quad& \alpha^2 - \beta \leq 0, \\
                    & \I - 2\alpha\overline{\lap} + \beta[\overline{\lap}^2+ 2\widetilde{\lap}] - \frac{1}{\nworkers}\one\one\tp \preceq \rho\I
\end{split}
\end{align}
where $\beta$ is an auxiliary variable, $\overline{\lap} = \sum_{j=1}^\nmatching p_j\lap_j$ and $\widetilde{\lap}=\sum_{j=1}^\nmatching p_j(1-p_j)\lap_j$.
\end{thm}
Similar to the optimization problem \Cref{eqn:opt_p}, the semi-definite programming problem \Cref{eqn:opt_a} needs to be solved only once at the beginning of training, and this additional computation time is negligible compared to the total training time.

Note that the activation probabilities in {\algo} implicitly influence the spectral norm as well. Ideally, one should jointly optimize $p_i$'s and $\alpha$ via a formulation like \Cref{eqn:opt_a}. However, the resulting optimization problem is non-convex and cannot be solved efficiently. Therefore, in {\algo}, we separately optimize the $p_i$'s and the parameter $\alpha$. Optimizing $p_i$'s via \Cref{eqn:opt_p} can be thought of as minimizing an upper bound of the spectral norm $\rho$. A more detailed discussion on this is given in Appendix B.

\section{Error Convergence Analysis}\label{sec:analysis}
In this section, we study how the optimization error convergence of {\algo} is affected by the choice of communication budget $C_b$, and how it compares to the convergence of {\dsgd}. In order to facilitate the analysis, we define the averaged iterate as $\avgx^{(k)} = \frac{1}{\nworkers}\sum_{i=1}^\nworkers \x_i^{(k)}$ and the lower bound of the objective function as $F_\text{inf}$. Since, we focus on general non-convex loss functions, the quantity of interest is the averaged gradient norm: $\frac{1}{K}\sum_{k=1}^K \Exs[\Arrowvert\nabla \obj(\avgx^{(k)})\Arrowvert^2]$. When it approaches zero, the algorithm converges to a stationary point. The analysis is centered around the following common assumptions.

\begin{assump}\label{assump:lip}
Each local objective function $F_i(\x)$ is differentiable and its gradient is $L$-Lipschitz: $\vecnorm{\tg_i(\x) -\tg_i(\mathbf{y})} \leq \lip \vecnorm{\x - \mathbf{y}}, \forall i \in \{1,2,\dots,\nworkers\}$.
\end{assump}
\begin{assump}\label{assump:unbiased}
Stochastic gradients at each worker node are unbiased estimates of the true gradient of the local objectives: $\Exs[\sg(\x_i^{(k)};\xi_i^{(k)})|\mathcal{F}^{(k)}] = \tg_i(\x_i^{(k)}), \forall i \in \{1,2,\dots,\nworkers\}$, where $\mathcal{F}^{(k)}$ denotes the 
sigma algebra generated by noise in the stochastic gradients and the graph activation probabilities until iteration $k$.
\end{assump}
\begin{assump}\label{assump:var}
The variance of stochastic gradients at each worker node is uniformly bounded: $\Exs[\Arrowvert\sg(\x_i^{(k)};\xi_i^{(k)}) - \tg_i(\x_i^{(k)})\Arrowvert^2 | \mathcal{F}^{(k)}] \leq \vbnd, \forall i \in \{1,2,\dots,\nworkers\}$.
\end{assump}
\begin{assump}\label{assump:diviation}
The deviation of local objectives' gradients are bounded by a non-negative constant: $\frac{1}{\nworkers}\sum_{i=1}^\nworkers \vecnorm{\tg_i(\x)-\tg(\x)}^2 \leq \zeta^2$.
\end{assump}

We first provide a non-asymptotic convergence guarantee for {\algo}, the proof of which is provided in Appendix C.
\begin{thm}[\textbf{Non-asymptotic Convergence of {\algo}}]
\label{thm:basic}
Suppose that all local models are initialized with $\avgx^{(1)}$ and $\{\mixmat^{(k)}\}_{k=1}^K$ is an i.i.d.\ mixing matrix sequence generated by {\algo}. Under Assumptions 1 to 4, and if the learning rate satisfies $\lr\lip \leq \min\{1, (\sqrt{\rho^{-1}}-1)/4\}$, where $\rho$ is the spectral norm (i.e., largest singular value) of matrix $\Exs[\mixmat^{(k)\top}\mixmat^{(k)}]-\frac{1}{\nworkers}\one\one\tp$, then after $K$ iterations,
\begin{align}
    \frac{1}{K}\sum_{k=1}^K &\Exs[\Arrowvert\nabla \obj(\avgx^{(k)})\Arrowvert^2 \leq \parenth{\frac{2[F(\avgx^{(1)})-F_{\text{inf}}]}{\lr K} + \frac{\lr\lip\vbnd}{\nworkers}}\frac{1}{1-2D}+\frac{2\lr^2\lip^2\rho}{1-\sqrt{\rho}}\parenth{\frac{\vbnd}{1+\sqrt{\rho}} + \frac{3\zeta^2}{1-\sqrt{\rho}}}\frac{1}{1-2D}
\end{align}
where $D=6\lr^2\lip^2\rho/(1-\sqrt{\rho})^2 < \frac{1}{2}$. In particular, it is guaranteed that $\rho < 1$ for arbitrary communication budget $C_b$.
\end{thm}
\textbf{Consistency with Previous Results.} Note that if $\mixmat^{(k)}=\one\one\tp/\nworkers$, then $\rho=0$ and the error bound in \Cref{thm:basic} reduces to that for fully synchronous SGD derived in \citep{bottou2016optimization}. When $\mixmat^{(k)}$ is fixed across iterations, then \Cref{thm:basic} reduces to the case of {\dsgd} and recovers the bound in \citep{lian2017can,wang2018cooperative}. \Cref{thm:basic} also reveals that the only difference in the optimization error upper bound between {\algo} and {\dsgd} is the value of spectral norm $\rho$. A smaller value of $\rho$ yields a lower optimization error bound.

\textbf{Dependence on Communication Budget $C_b$.} 
While it is difficult to get the analytical form of $\rho$ in terms of the communication budget $C_b$, in \Cref{fig:sim}, we present some numerical results obtained by solving the optimization problems \Cref{eqn:opt_p,eqn:opt_a}. Recall that a lower spectral norm $\rho$ means better error-convergence in terms of iterations. Observe that \textsc{Matcha} takes $2-3\times$ less communication time while preserving the same spectral norm as {\dsgd}. By setting a proper communication budget (for instance $C_b \approx 0.5$ in \Cref{fig:g2}), {\algo} can have even lower spectral norm than vanilla DecenSGD. Besides, to achieve the same spectral norm, {\algo} always requires much less communication budget than periodic DecenSGD. These theoretical findings are corroborated by extensive experiments in \Cref{sec:exps}. 

Furthermore, if the learning rate is configured properly, {\algo} can achieve a linear speedup in terms of number of worker nodes, matching the same rate as {\dsgd} and fully synchronous SGD.
\begin{corollary}[\textbf{Linear speedup}]\label{corollary:basic}
Under the same conditions as \Cref{thm:basic}, if the learning rate is set as $\lr = \sqrt{\frac{\nworkers}{K}}$, then after total $K$ iterations, we have
\begin{align}
    \frac{1}{K}\sum_{k=1}^K\Exs\brackets{\vecnorm{\tg(\avgx^{(k)})}^2}
    &=\mathcal{O}\parenth{\frac{1}{\sqrt{\nworkers K}}} +\mathcal{O}\parenth{ \frac{\nworkers}{K}} \label{eqn:rate}
\end{align}
where all the other constants are subsumed in $\mathcal{O}$.
\end{corollary}
It is worth noting that when the total iteration $K$ is sufficiently large ($K \geq m^3$), the convergence of {\algo} will be dominated by the first term $1/\sqrt{mK}$.
\begin{rem}
While \Cref{thm:basic} and \Cref{corollary:basic} focus on the convergence guarantee for {\algo}, the results can be general and hold for any arbitrary $\{\mixmat^{(k)}\}$ sequence as long as $\rho = \opnorm{\Exs[\mixmat^{(k)\top}\mixmat^{(k)}]-\frac{1}{\nworkers}\one\one\tp} < 1$.
\end{rem}

\section{Experimental Results}\label{sec:exps}
We evaluate the performance of {\algo} in multiple deep learning tasks: (1) image classification on CIFAR-10 and CIFAR-100 \citep{krizhevsky2009learning}; (2) Language modeling on Penn Treebank corpus (PTB) dataset \citep{marcus1993building}. All training datasets are evenly partitioned over a network of workers. All algorithms are trained for sufficiently long time until convergence or overfitting. The learning rate is fine-tuned for {\dsgd} and then used for all other algorithms, since we treat {\algo} as an in-place replacement of {\dsgd}. Note that this will in fact disadvantage {\algo} in the comparison with {\dsgd}. A detailed description of the training configurations is provided in Appendix A.1.

\begin{figure}
    \centering
    \begin{subfigure}{.3\textwidth}
    \centering
    \includegraphics[width=\textwidth]{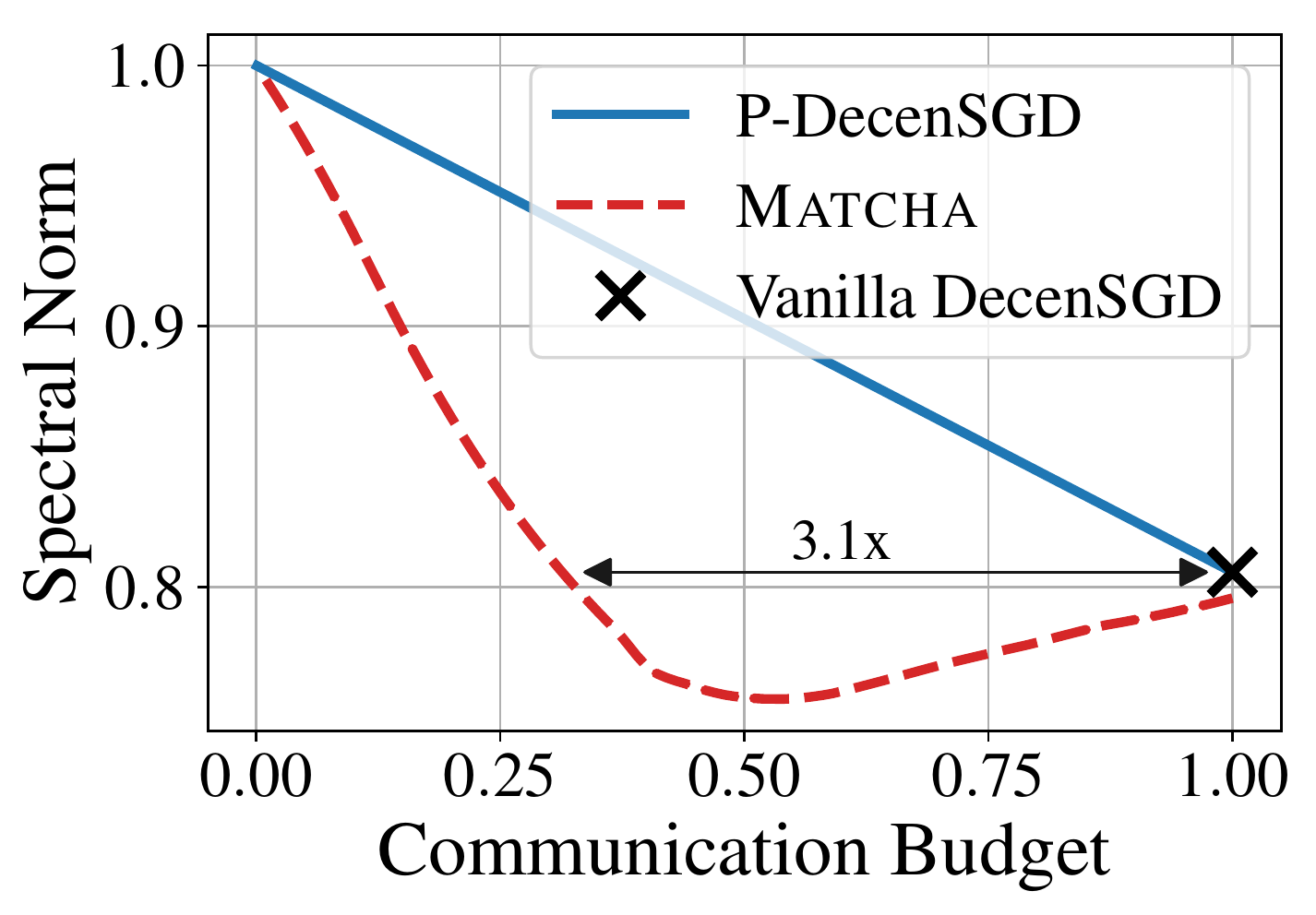}
    \includegraphics[width=.8\textwidth]{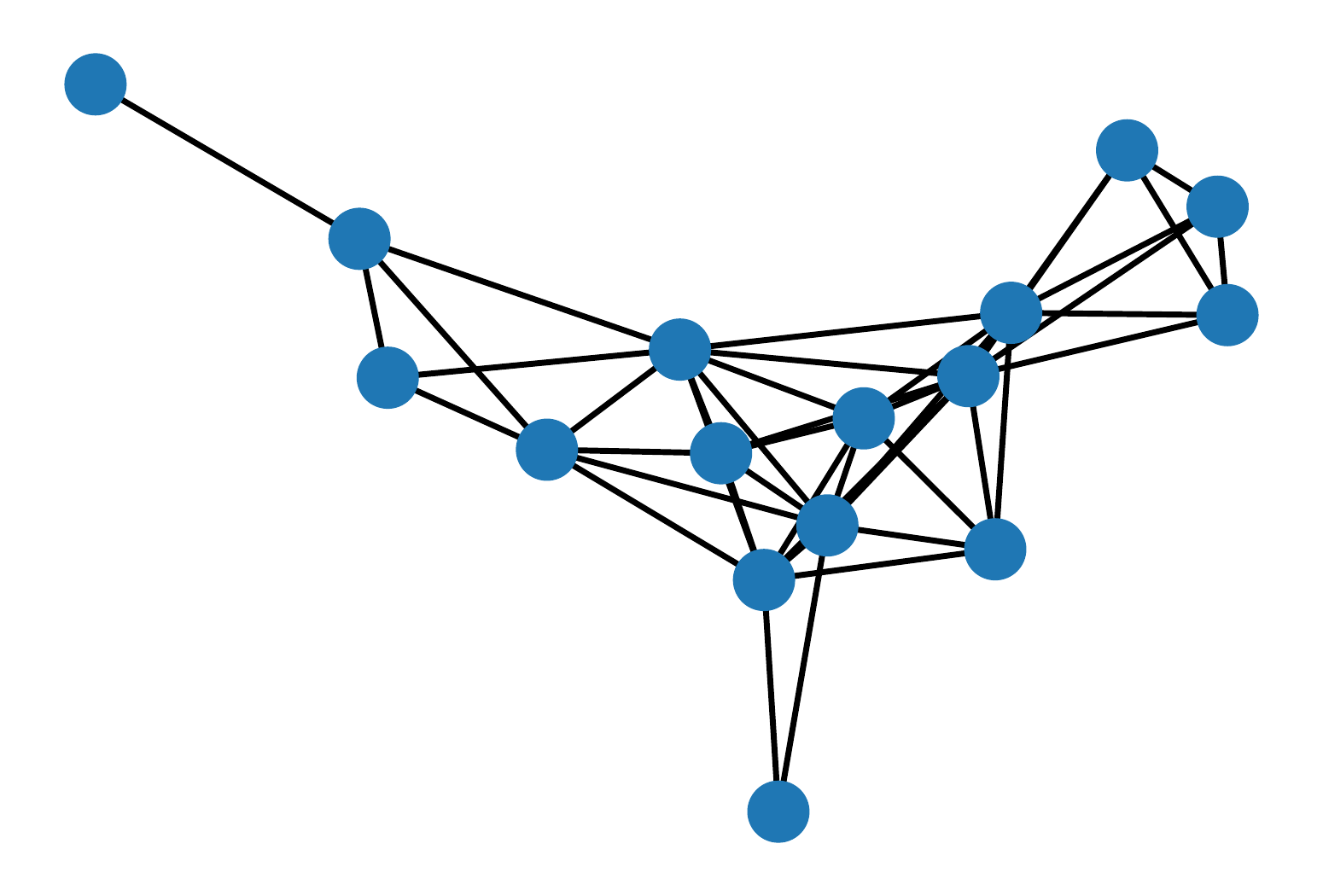}
    \caption{Geometric graph.}
    \label{fig:g2}
    \end{subfigure}%
    ~
    \begin{subfigure}{.3\textwidth}
    \centering
    \includegraphics[width=\textwidth]{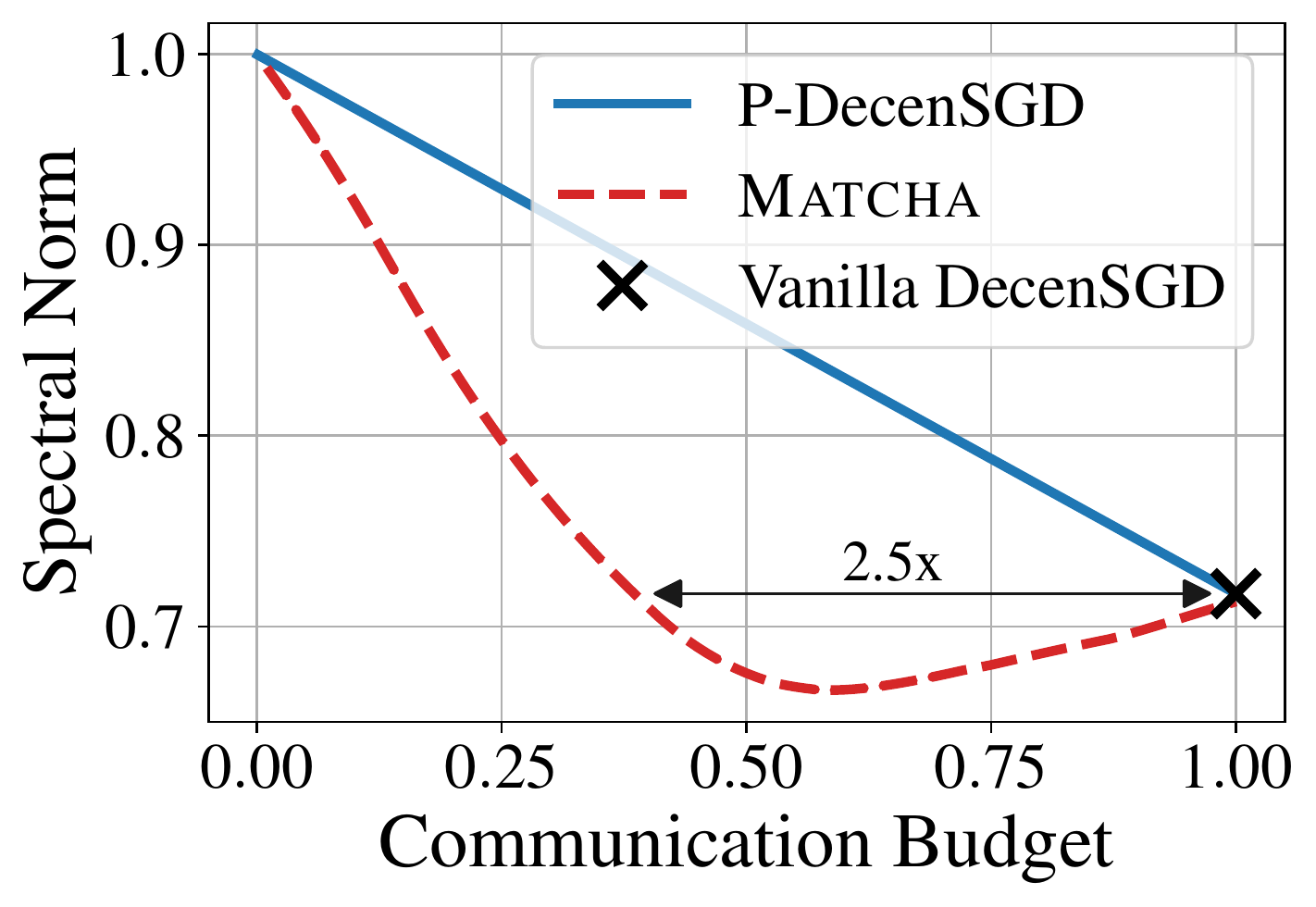}
    \includegraphics[width=.8\textwidth]{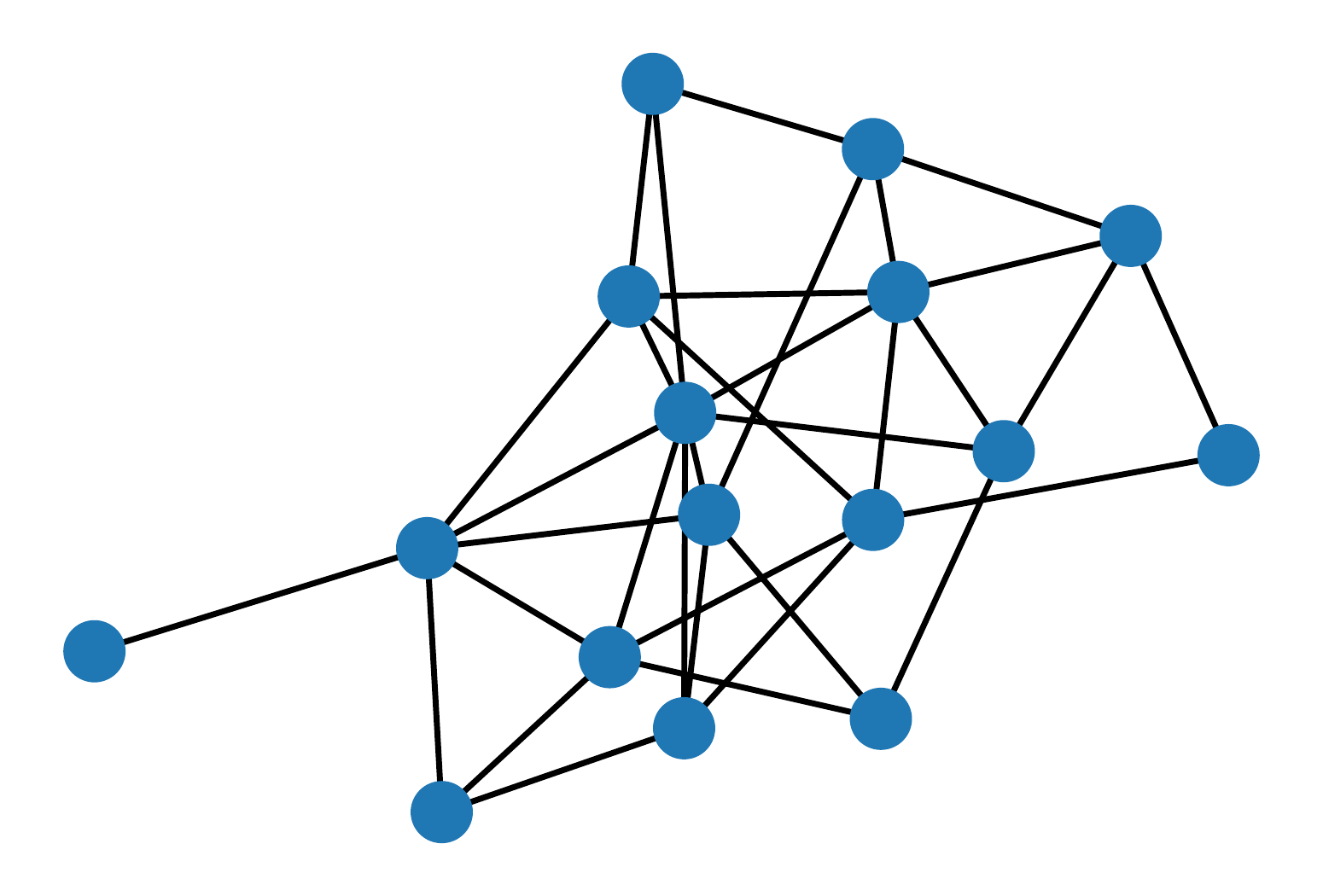}
    \caption{Erd\H{o}s-R\'enyi's graph.}
    \label{fig:g3}
    \end{subfigure}%
    \caption{Examples on how the spectral norm $\rho$ varies over communication budget in \textsc{Matcha}. In both (a) and (b), there are $16$ worker nodes. {\algo} typically costs $2-3\times$ less communication time than {\dsgd} (black crosses) while maintaining the exactly same or even lower value of $\rho$ (i.e., same or better error upper bound).}
    \label{fig:sim}
    \vspace{-6pt}
\end{figure}

\begin{figure*}[!hbt]
    \centering
    \begin{subfigure}{.3\textwidth}
    \centering
    \includegraphics[width=\textwidth]{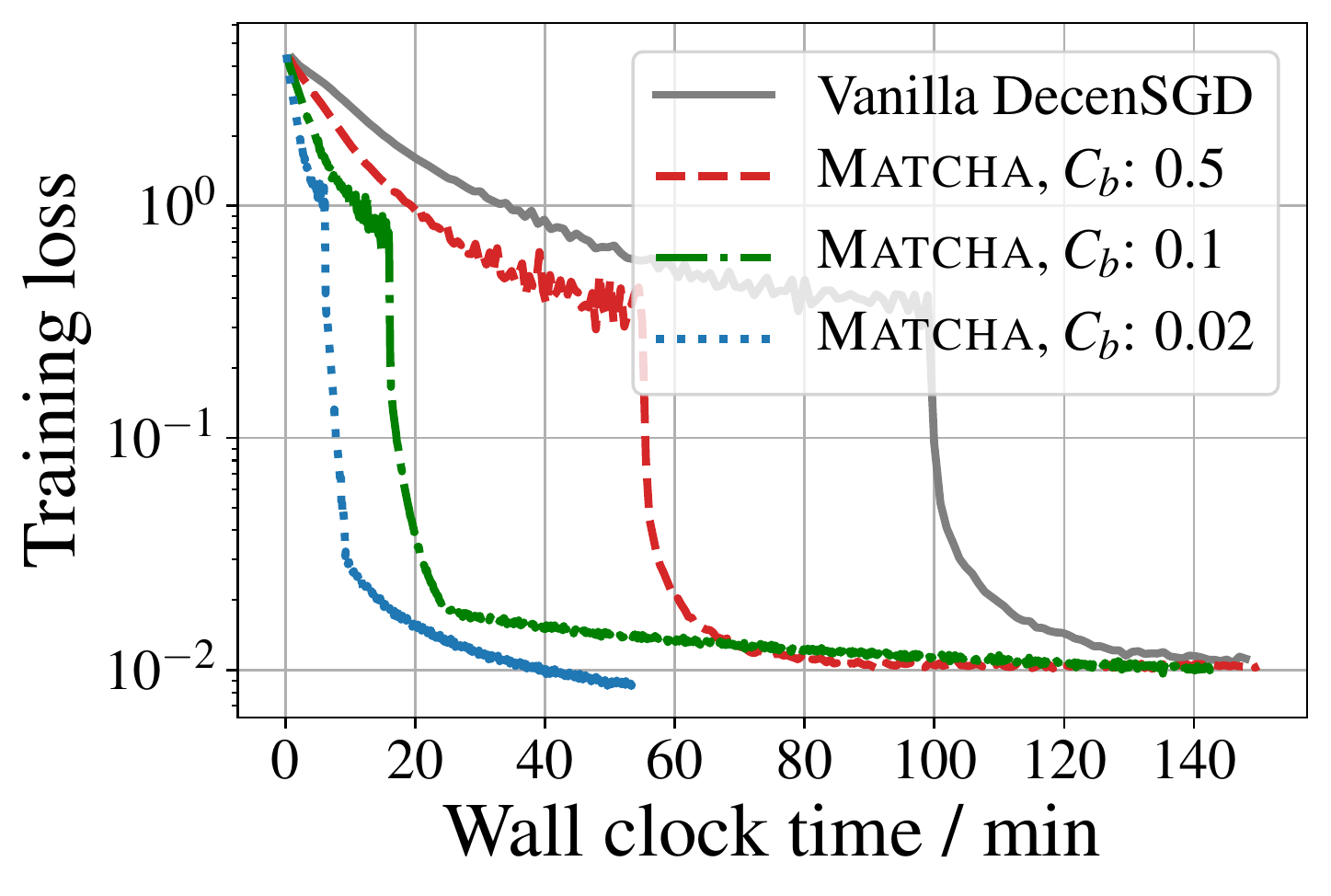}
    \caption{WideResNet on CIFAR-100.}
    \label{fig:time_comp_c100}
    \end{subfigure}%
    ~
    \begin{subfigure}{.3\textwidth}
    \centering
    \includegraphics[width=\textwidth]{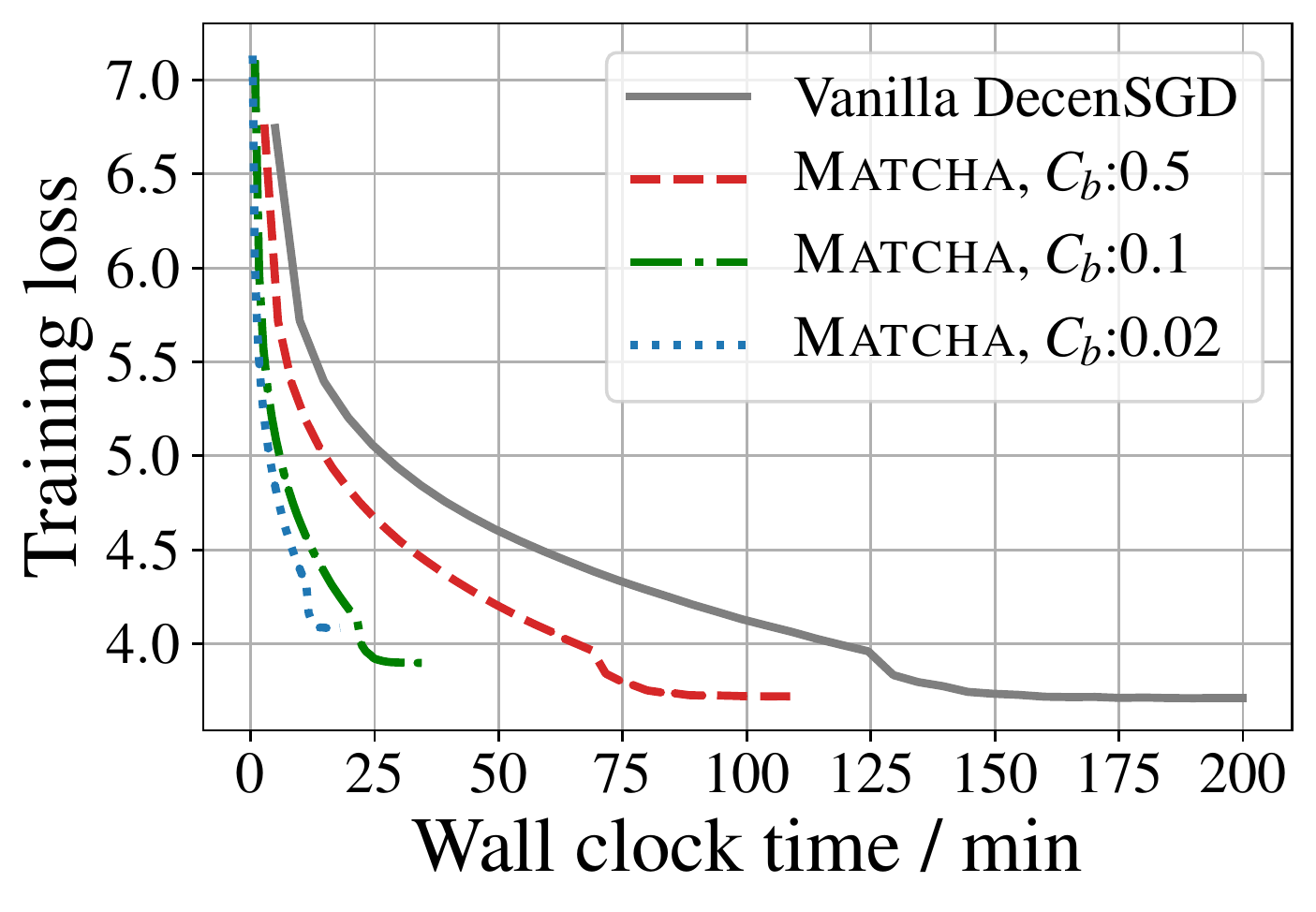}
    \caption{LSTM on Penn Treebank.}
    \label{fig:time_comp_ptb}
    \end{subfigure}%
    ~
    \begin{subfigure}{.3\textwidth}
    \centering
    \includegraphics[width=\textwidth]{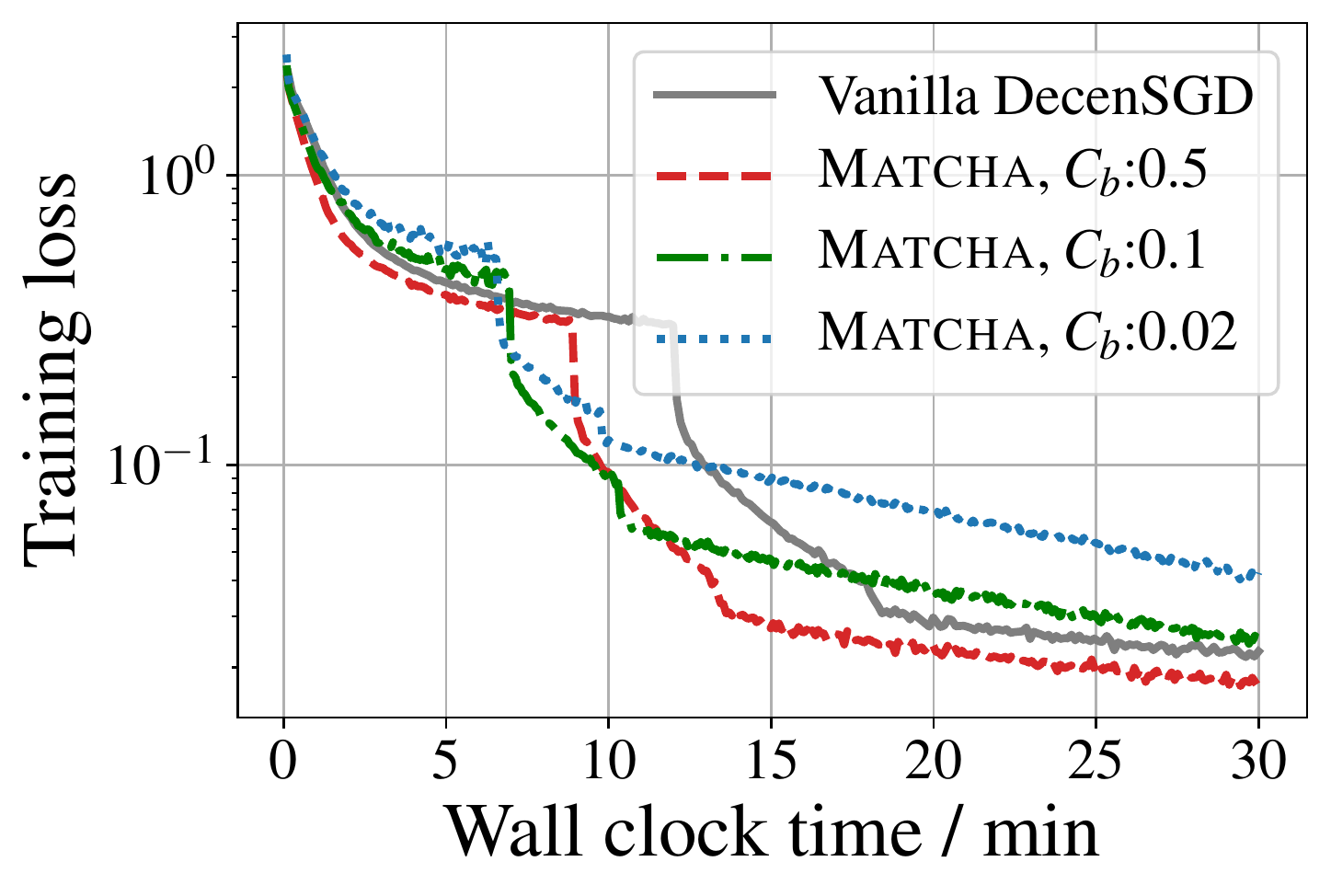}
    \caption{ResNet on CIFAR-10.}
    \label{fig:time_comp_c10}
    \end{subfigure}%
    
    \bigskip
    \begin{subfigure}{.3\textwidth}
    \centering
    \includegraphics[width=\textwidth]{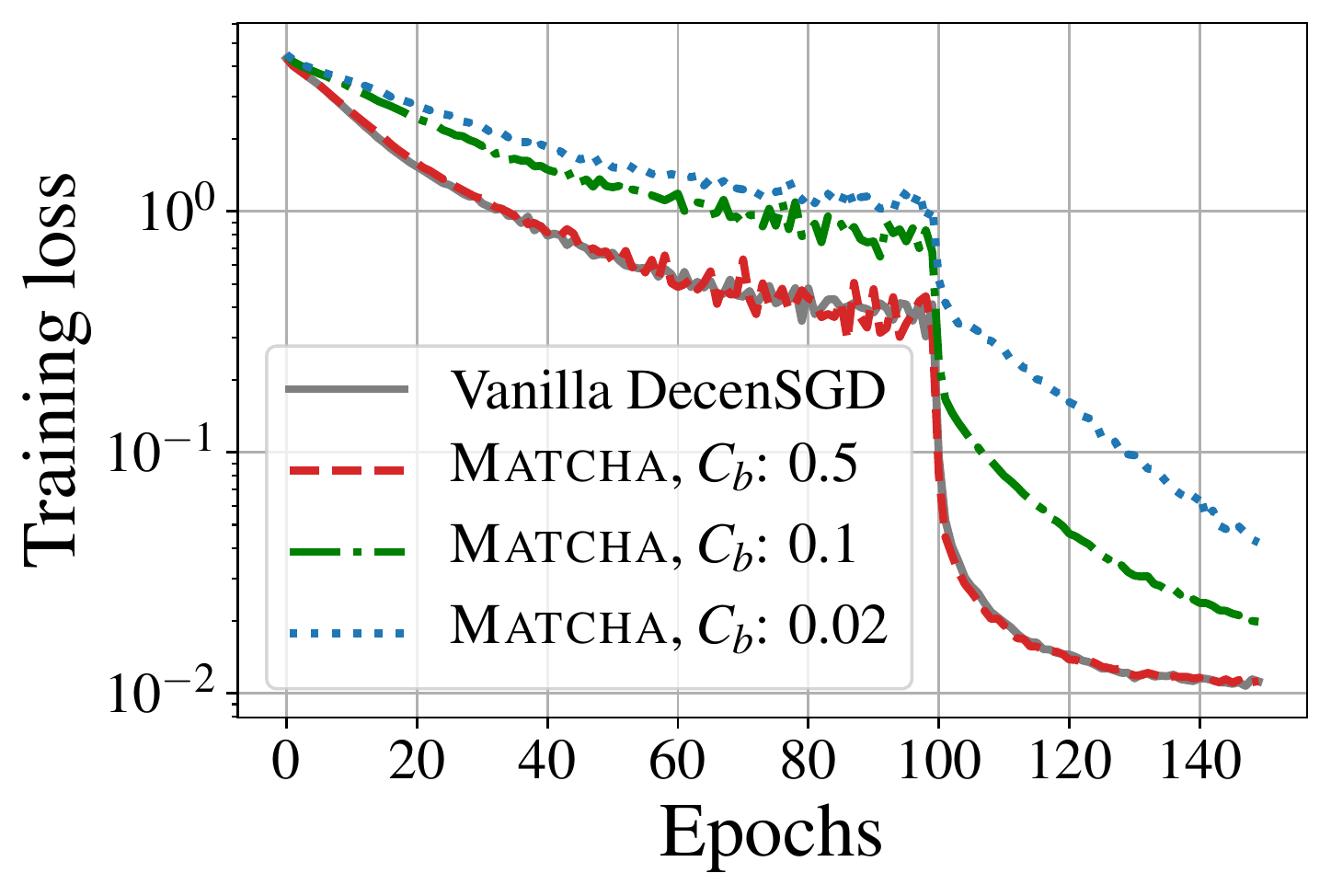}
    \caption{WideResNet on CIFAR-100.}
    \label{fig:epoch_comp_c100}
    \end{subfigure}%
    ~
    \begin{subfigure}{.3\textwidth}
    \centering
    \includegraphics[width=\textwidth]{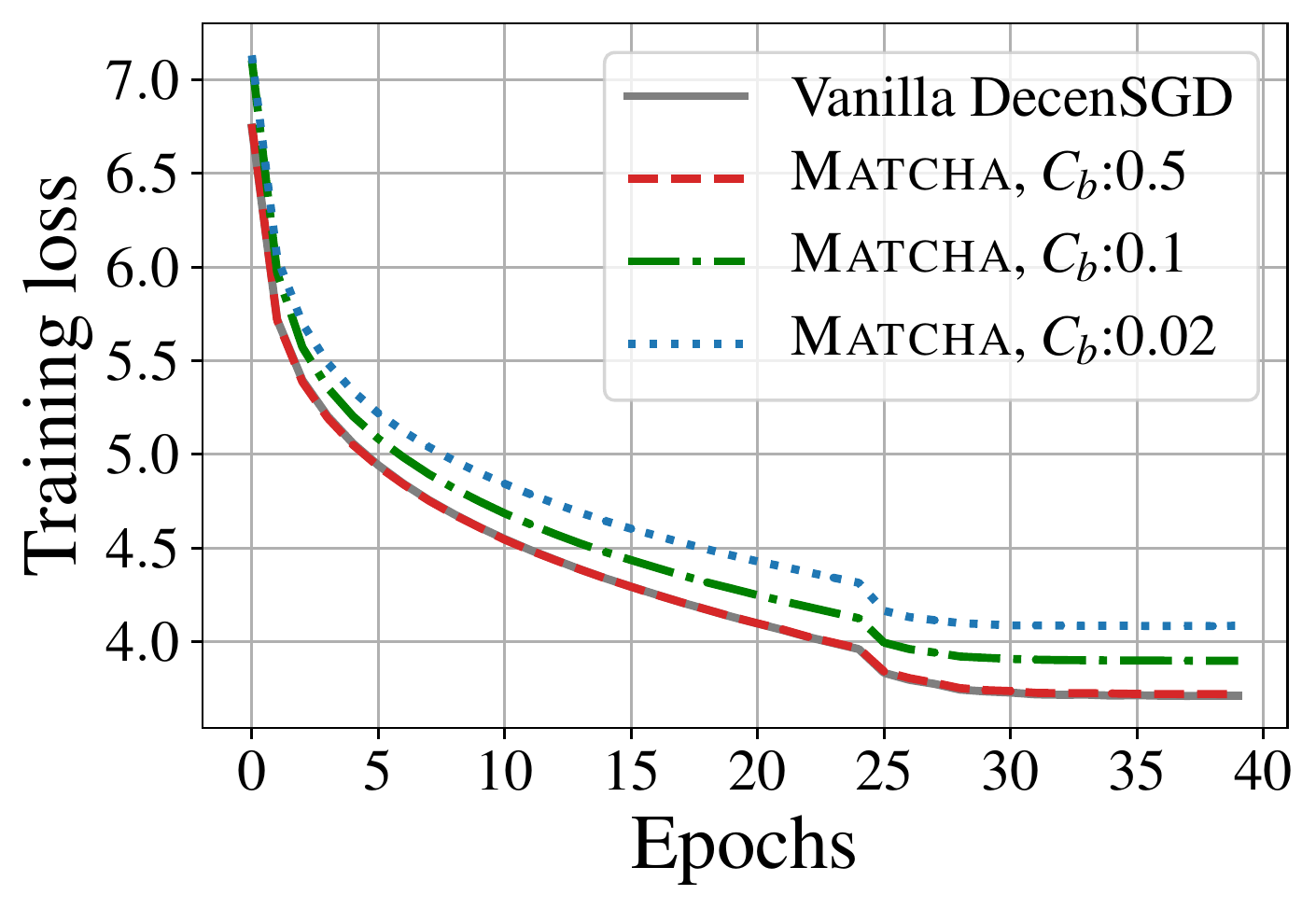}
    \caption{LSTM on Penn Treebank.}
    \label{fig:epoch_comp_ptb}
    \end{subfigure}%
    ~
    \begin{subfigure}{.3\textwidth}
    \centering
    \includegraphics[width=\textwidth]{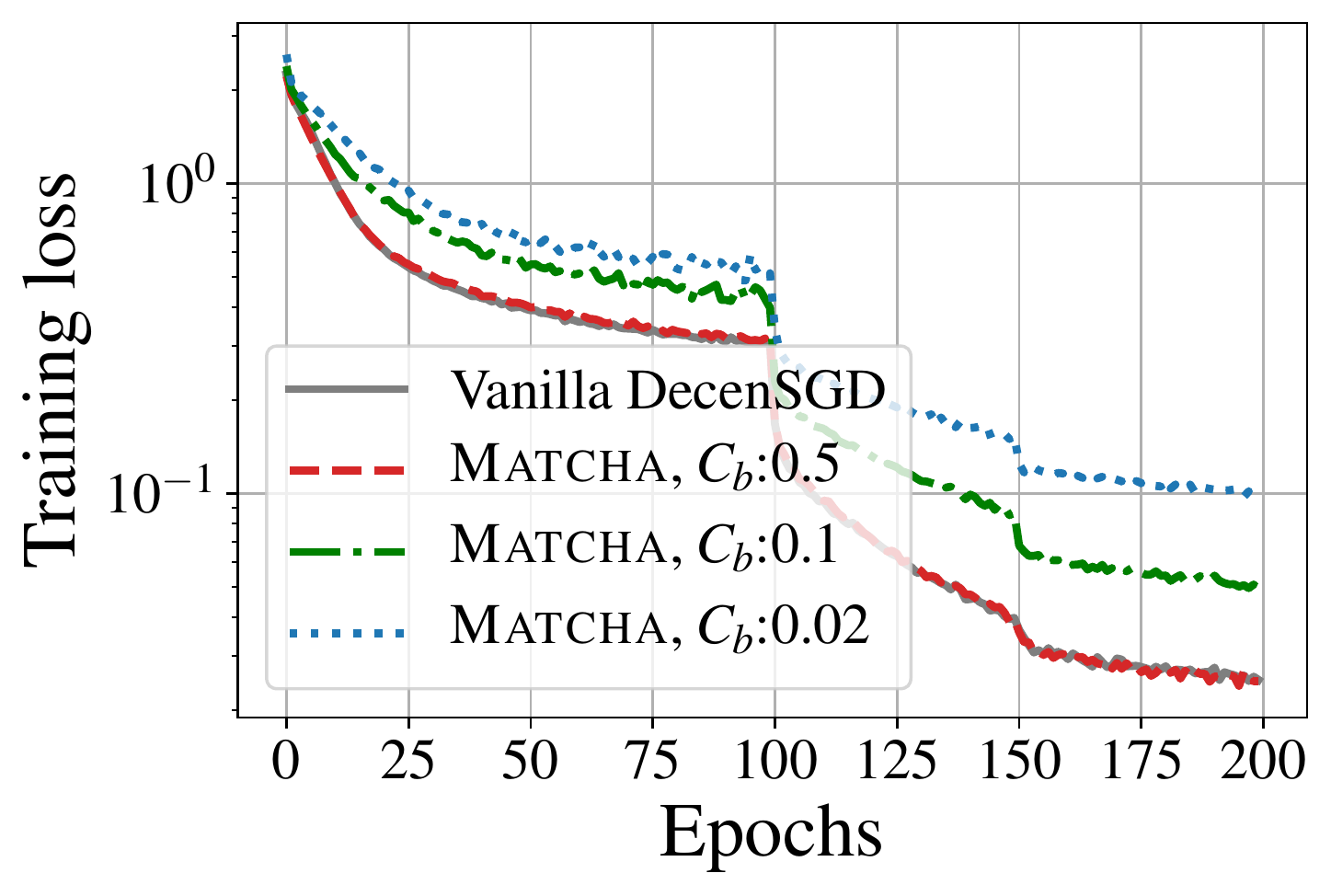}
    \caption{ResNet on CIFAR-10.}
    \label{fig:epoch_comp_c10}
    \end{subfigure}%
    \caption{Varying communication budgets $C_b$ in {\algo}. The base communication topology is \Cref{fig:base_topo} (an Erd\H{o}s-R\'enyi graph with $8$ nodes). As predicted by the theoretical result in \Cref{fig:g1_rho}, when the communication budget is $0.5$, {\algo} has nearly the same loss-versus-iteration curves as {\dsgd} but requires only half communication time per iteration. The corresponding test accuracy curves are presented in Appendix A.}
    \label{fig:cb}
    \vspace{-6pt}
\end{figure*}
\begin{figure*}[!htb]
    \centering
    \begin{subfigure}{.3\textwidth}
    \centering
    \includegraphics[width=\textwidth]{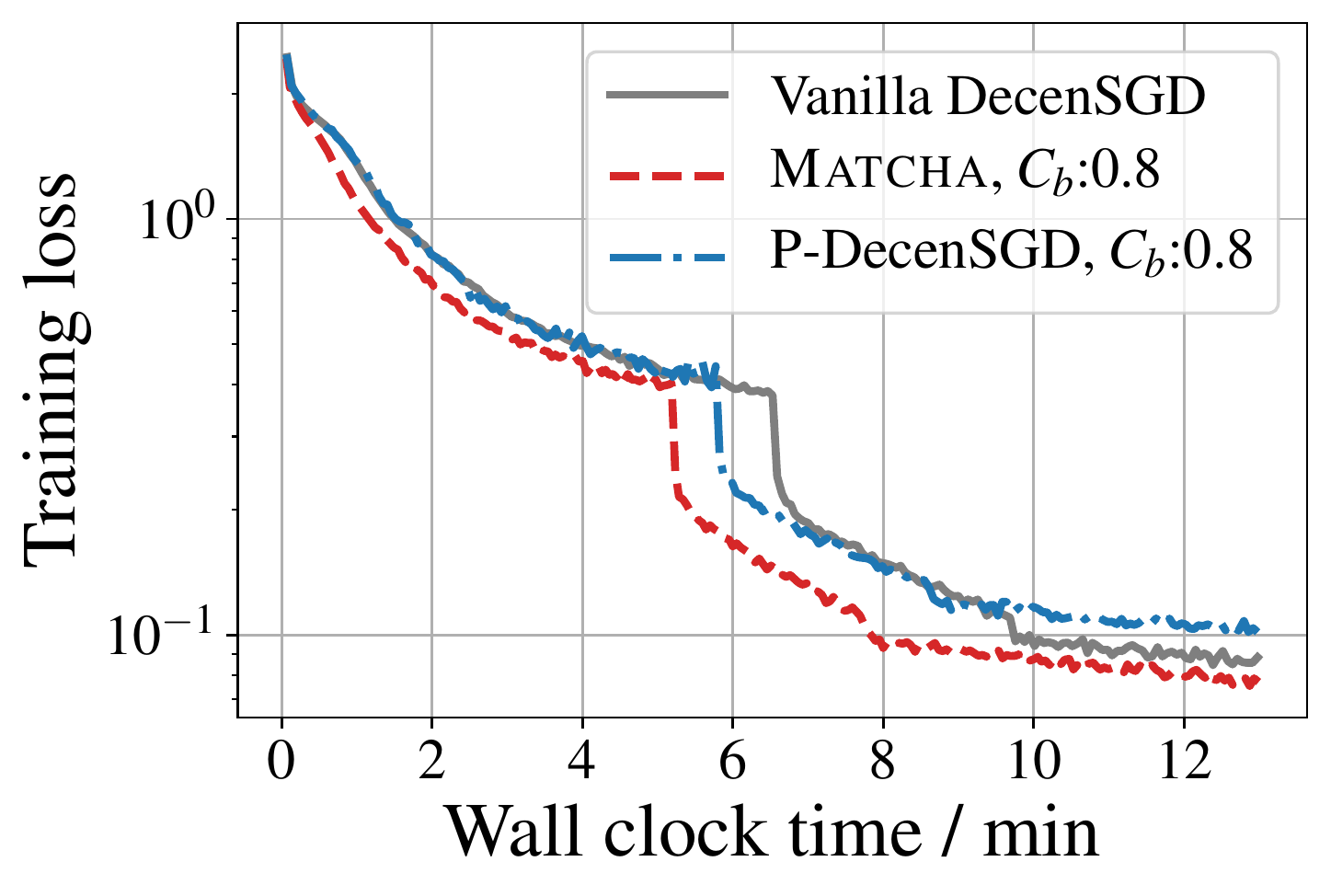}
    \caption{Maximal degree is $5$.}
    \label{fig:res_n16_sparse}
    \end{subfigure}%
    ~
    \begin{subfigure}{.3\textwidth}
    \centering
    \includegraphics[width=\textwidth]{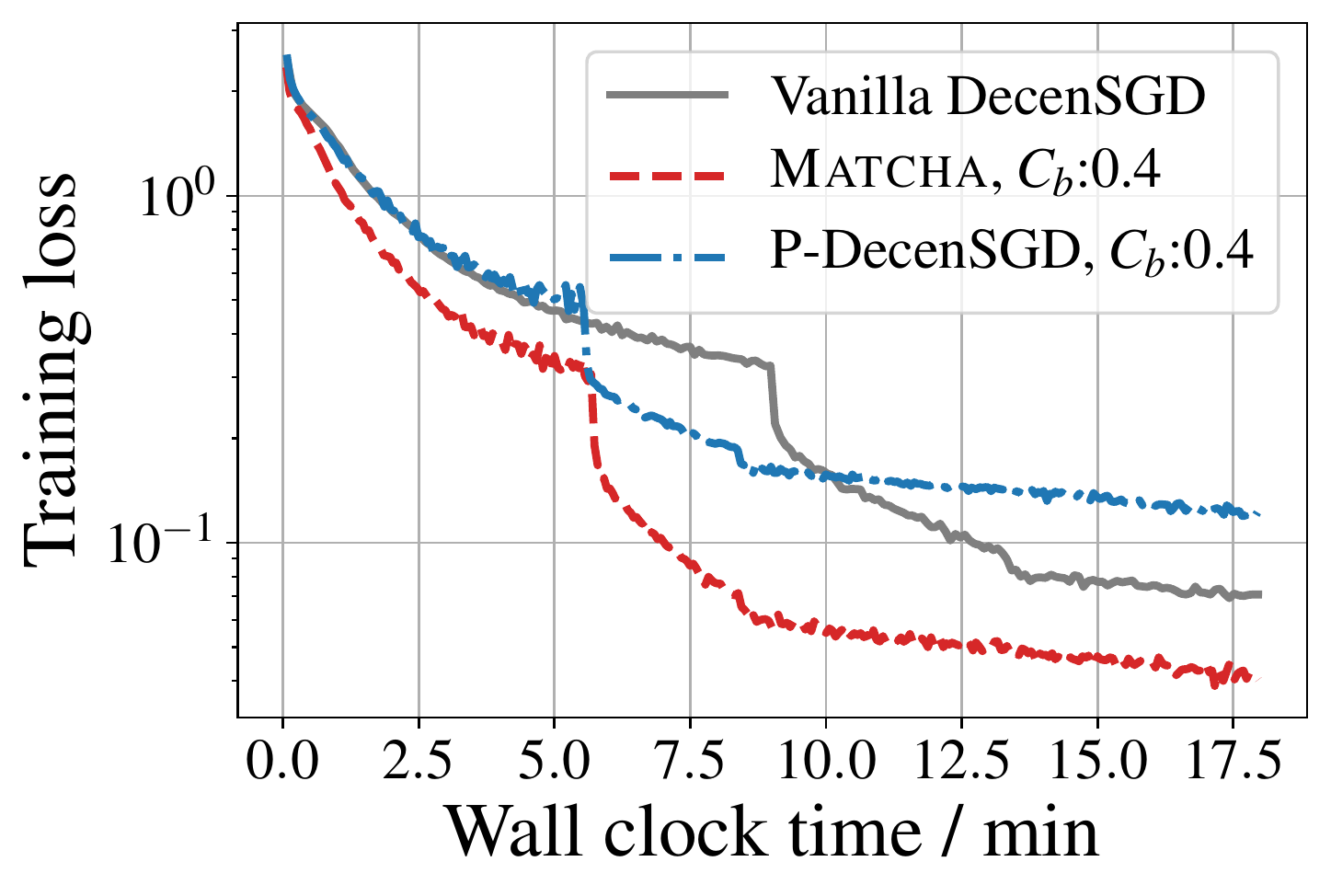}
    \caption{Maximal degree is $10$.}
    \label{fig:res_n16_mild}
    \end{subfigure}%
    ~
    \begin{subfigure}{.3\textwidth}
    \centering
    \includegraphics[width=\textwidth]{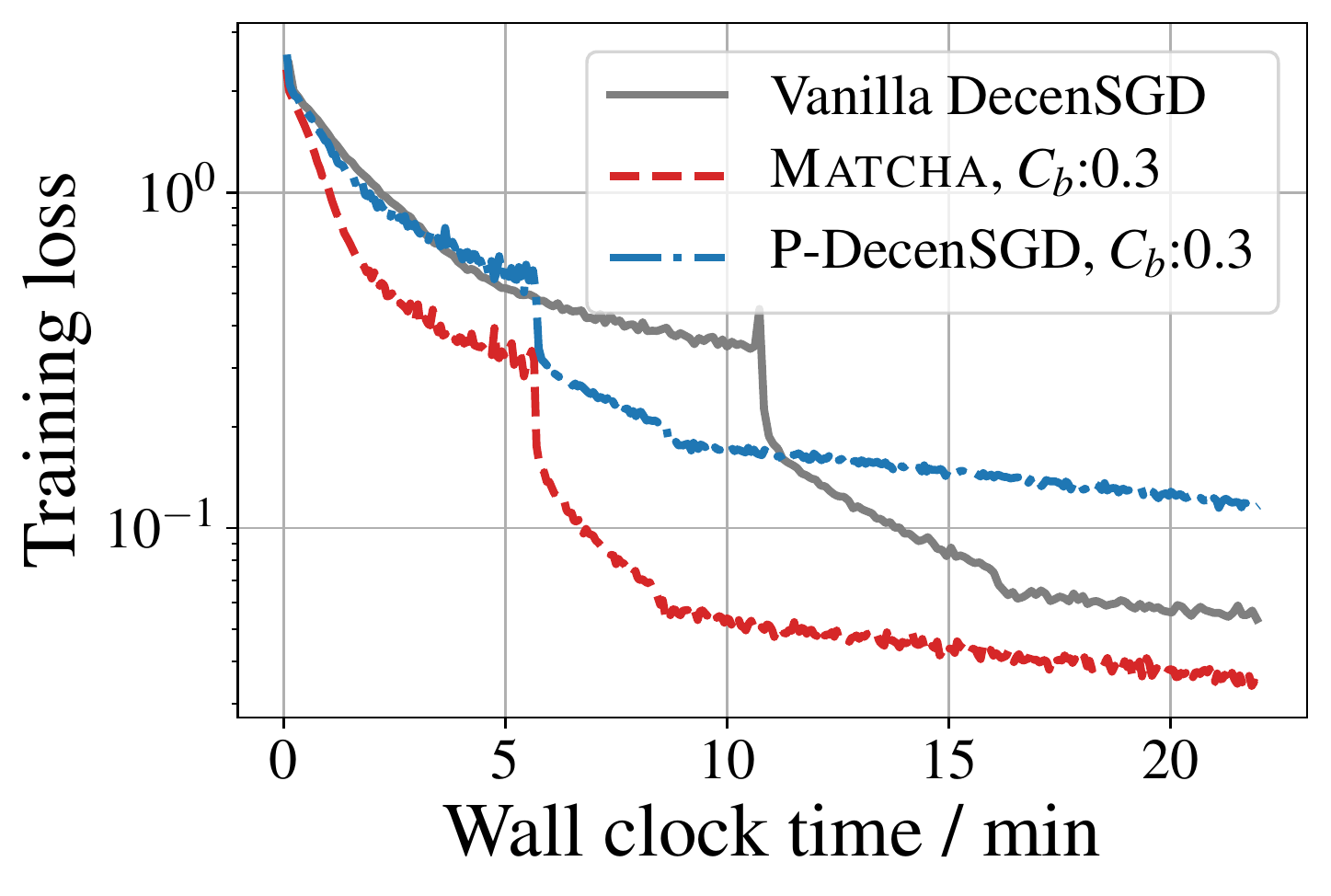}
    \caption{Maximal degree is $13$.}
    \label{fig:res_n16_dense}
    \end{subfigure}%
    
    \bigskip
    \begin{subfigure}{.3\textwidth}
    \centering
    \includegraphics[width=\textwidth]{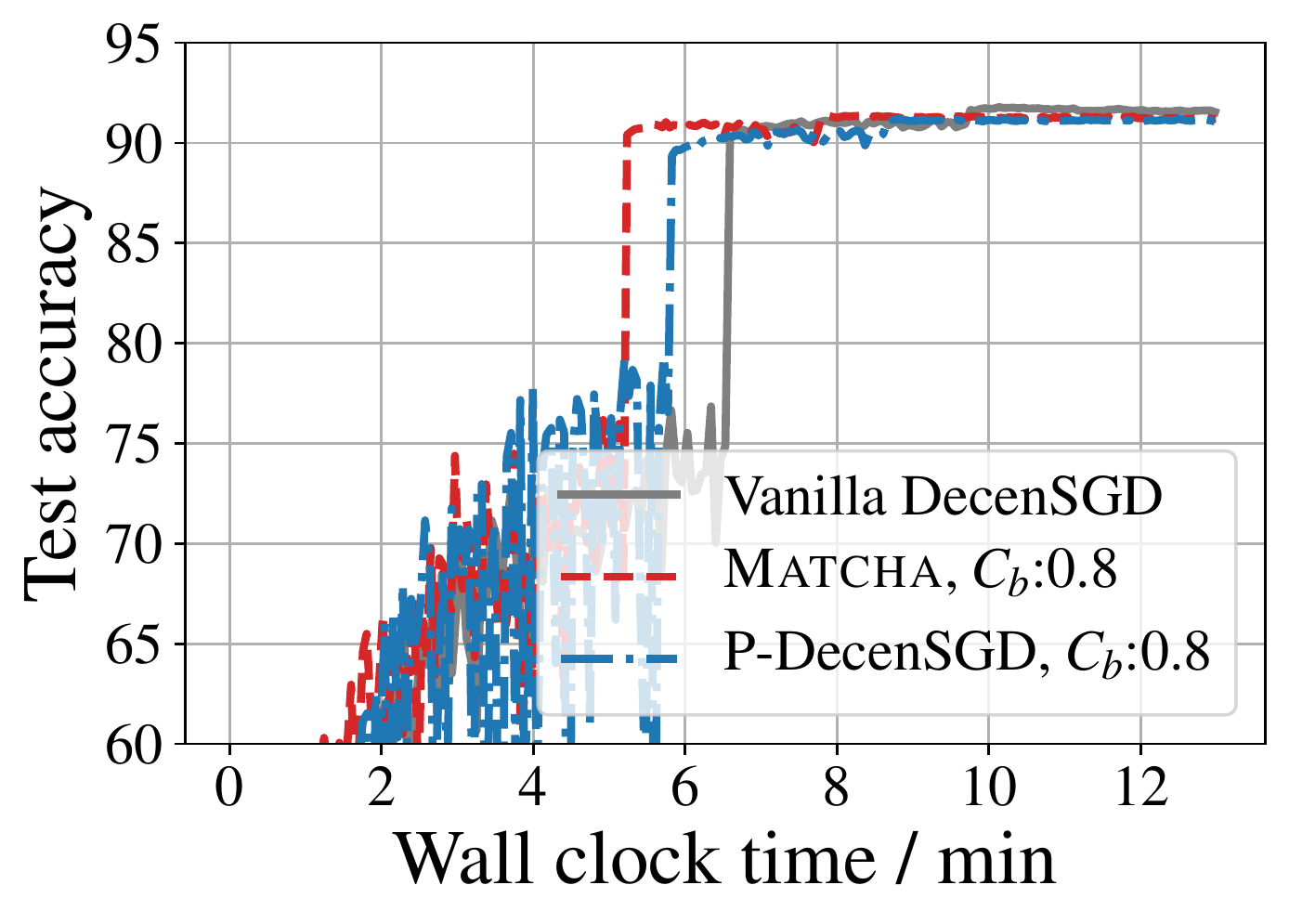}
    \caption{Maximal degree is $5$.}
    \label{fig:acc_sparse}
    \end{subfigure}%
    ~
    \begin{subfigure}{.3\textwidth}
    \centering
    \includegraphics[width=\textwidth]{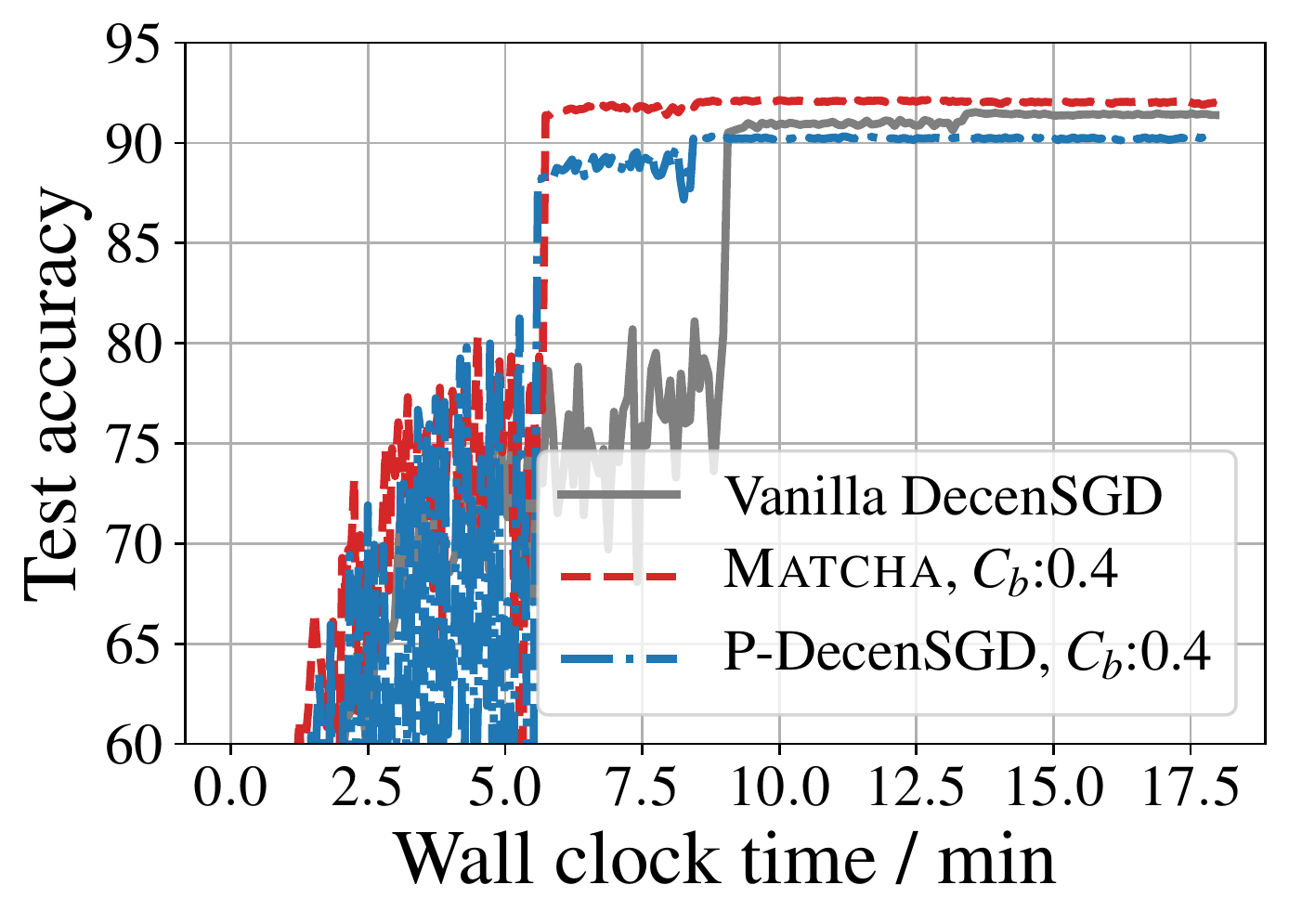}
    \caption{Maximal degree is $10$.}
    \label{fig:acc_mild}
    \end{subfigure}%
    ~
    \begin{subfigure}{.3\textwidth}
    \centering
    \includegraphics[width=\textwidth]{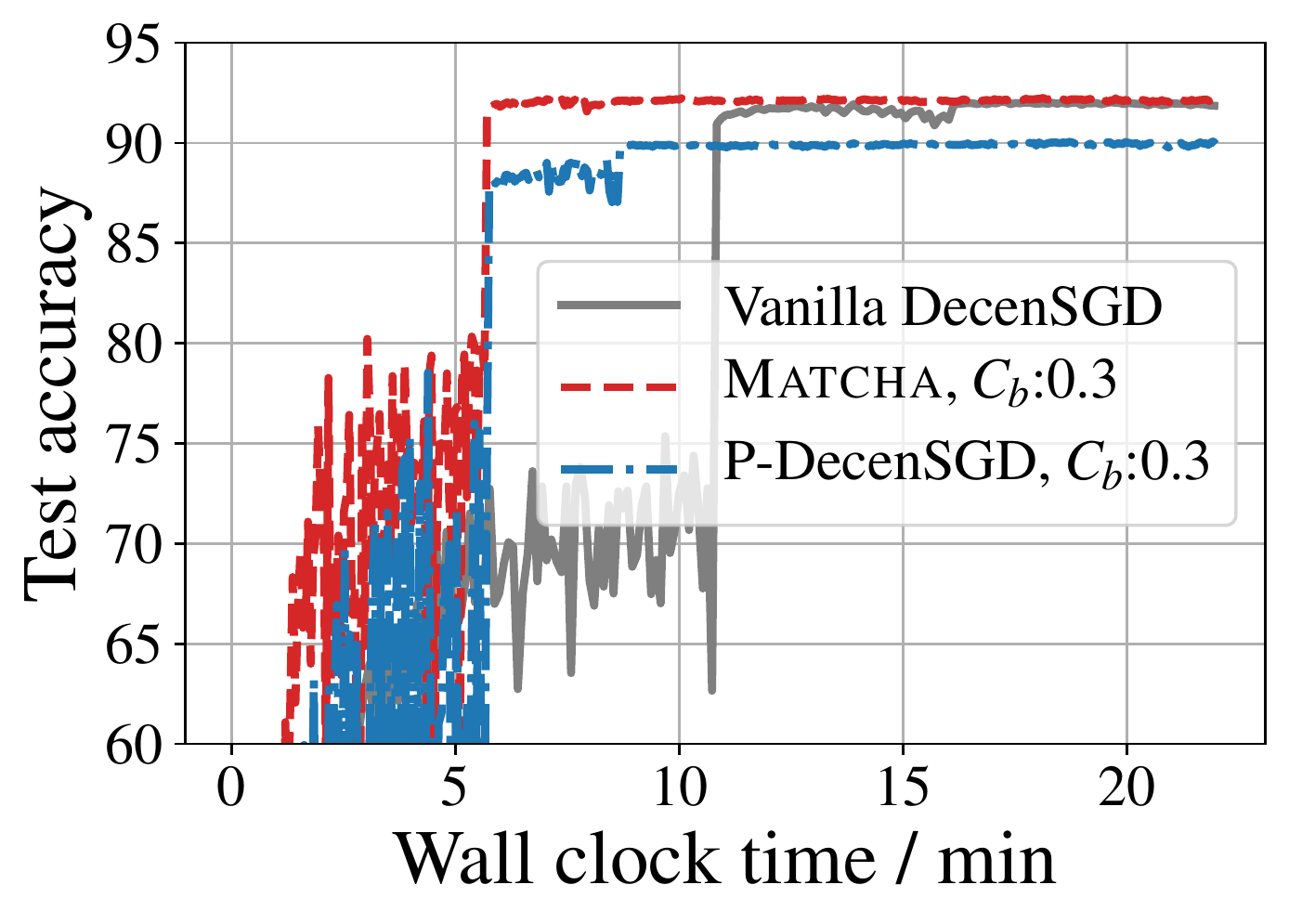}
    \caption{Maximal degree is $13$.}
    \label{fig:acc_dense}
    \end{subfigure}%
    \caption{Performance comparison on three random-generated geometric graphs with $16$ nodes and  different levels of connectivity. We train a ResNet-50 on CIFAR-10 dataset. The corresponding base typologies can be found in Appendix A. In particular, the base topology used in \Cref{fig:res_n16_mild,fig:acc_mild} is the same as \Cref{fig:g2}, which suggests that {\algo} can achieve lower spectral norm than {\dsgd} when $C_b \geq 0.3$. As a consequence, \Cref{fig:res_n16_mild,fig:acc_mild} show that {\algo} can reach a lower training loss and a higher test accuracy than {\dsgd}. }
    \label{fig:res_n16}
    \vspace{-6pt}
\end{figure*}

\textbf{Effectiveness of {\algo}.}
In \Cref{fig:cb}, we evaluate the performance of {\algo} with various communication budgets ($2\%, 10\%, 50\%$ of vanilla DecenSGD). The base communication topology is shown in \Cref{fig:illustration}. From \Cref{fig:epoch_comp_c100,fig:epoch_comp_ptb,fig:epoch_comp_c10}, one can observe that when the communication budget is set to $0.5$ (reducing expected communication time per iteration by $50$\%), {\algo} has the nearly identical training losses as {\dsgd} at every epoch. This empirical finding reinforces the claim  in Section 5 regarding the similarity of the algorithms' performance in terms of iterations. When we continue to decrease the communication budget, {\algo} attains significantly faster convergence with respect to wall-clock time in communication-intensive tasks. In particular, on CIFAR-100 (see \Cref{fig:time_comp_c100}), {\algo} with $C_b=0.02$ can take $5\times$ less wall-clock time than {\dsgd} to reach a training loss of $0.1$.

\textbf{Effects of Base Communication Topology.}
To further verify the applicability of {\algo} to arbitrary graphs, we evaluate it on different topologies with varying levels of connectivity. In \Cref{fig:res_n16}, we present experimental results on three random geometric graph topologies that have different maximal degrees. In particular, when the maximal degree is $10$ (see \Cref{fig:res_n16_mild}), {\algo} with communication budget $C_b=0.4$ not only reduces the mean communication time per iteration by $1/0.4=2.5\times$ but also has lower error than {\dsgd}. This result corroborates the corresponding spectral norm versus communication budget curve shown in \Cref{fig:g2}. When we further increase the density of the base topology (see \Cref{fig:res_n16_dense}), {\algo} can reduce communication time per iteration by $1/0.3\simeq 3.3\times$ without hurting the error-convergence.

Another interesting observation is that {\algo} gives a more drastic communication reduction for denser base graphs. As shown in \Cref{fig:res_n16}, along with the increase in the density of the base graph, the training time of {\dsgd} also increases from $13$ to $22$ minutes to finish $200$ epochs. However, in {\algo}, since the effective maximal degree in all cases is maintained to be about $4$ by controlling communication budget, the total training time of $200$ epochs remains nearly the same (about $11$ minutes). Moreover, {\algo} also takes less and less time to achieve a training loss of $0.1$, on the contrast to {\dsgd} and P-DecenSGD.

\textbf{Comparison to Periodic DecenSGD.}
As discussed in Section 3, a naive way to reduce the communication time per iteration is to use the whole base graph for synchronization after every few iterations \citep{wang2018cooperative,tsianos2012communication}. Instead, in {\algo}, we allow different matchings to have different communication frequencies. Similar to the theoretical simulations in \Cref{fig:sim}, the results in \Cref{fig:res_n16} show that given a fixed communication budget $C_b$, {\algo} consistently outperforms periodic DecenSGD. More results are presented in the Appendix.

\section{Concluding Remarks}
In this paper, we propose {\algo} to reduce the communication delay of decentralized SGD over an arbitrary node topology. The key idea in {\algo} is that workers communicate over the connectivity-critical links with higher probability, which is achieved via matching decomposition sampling. Rigorous theoretical analysis and experimental results show that {\algo} can reduce the communication delay while maintaining the same (or even improve) error-versus-iterations convergence rate. Future directions include adaptively changing the communication budget similar to \citep{Wang2018Adaptive}, extending {\algo} to directed communication graphs \citep{assran2018stochastic} and other decentralized computing algorithms.

\bibliographystyle{plain}
\bibliography{example_paper.bib}

\input{new_proof.tex}

\end{document}

%% file: new_proof.tex
\onecolumn
\appendix

\section{More Experimental Results}
\subsection{Detailed Experimental Setting}
\label{subsec:expt}
\textbf{Image Classification Tasks.} 
CIFAR-10 and CIFAR-100 consist of $60,000$ color images in $10$ and $100$ classes, respectively. For CIFAR-10 and CIFAR-100 training, we set the initial learning rate as $0.8$ and it decays by $10$ after $100$ and $150$ epochs. The mini-batch size per worker node is $64$. We train vanilla DecenSGD for $200$ epochs and all other algorithms for the same wall-clock time as vanilla DecenSGD.

\textbf{Language Model Task.} The PTB dataset contains $923,000$ training words. A two-layer LSTM with $1500$ hidden nodes in each layer \cite{press2016using} is adopted. We set the initial learning rate as $40$ and it decays by $4$ when the training procedure saturates. The mini-batch size per worker node is $10$. The embedding size is $1500$. All algorithms are trained for $40$ epochs.

\textbf{Machines.}
The training procedure is performed in a network of nodes, each of which is equipped with one NVIDIA TitanX Maxwell GPU and has a $40$ Gbps ($5000$ MB/s) Ethernet interface. \textsc{Matcha} is implemented with PyTorch and MPI4Py.

\subsection{More Results}
\begin{figure}[!ht]
    \centering
    \begin{subfigure}{.3\textwidth}
    \centering
    \includegraphics[width=\textwidth]{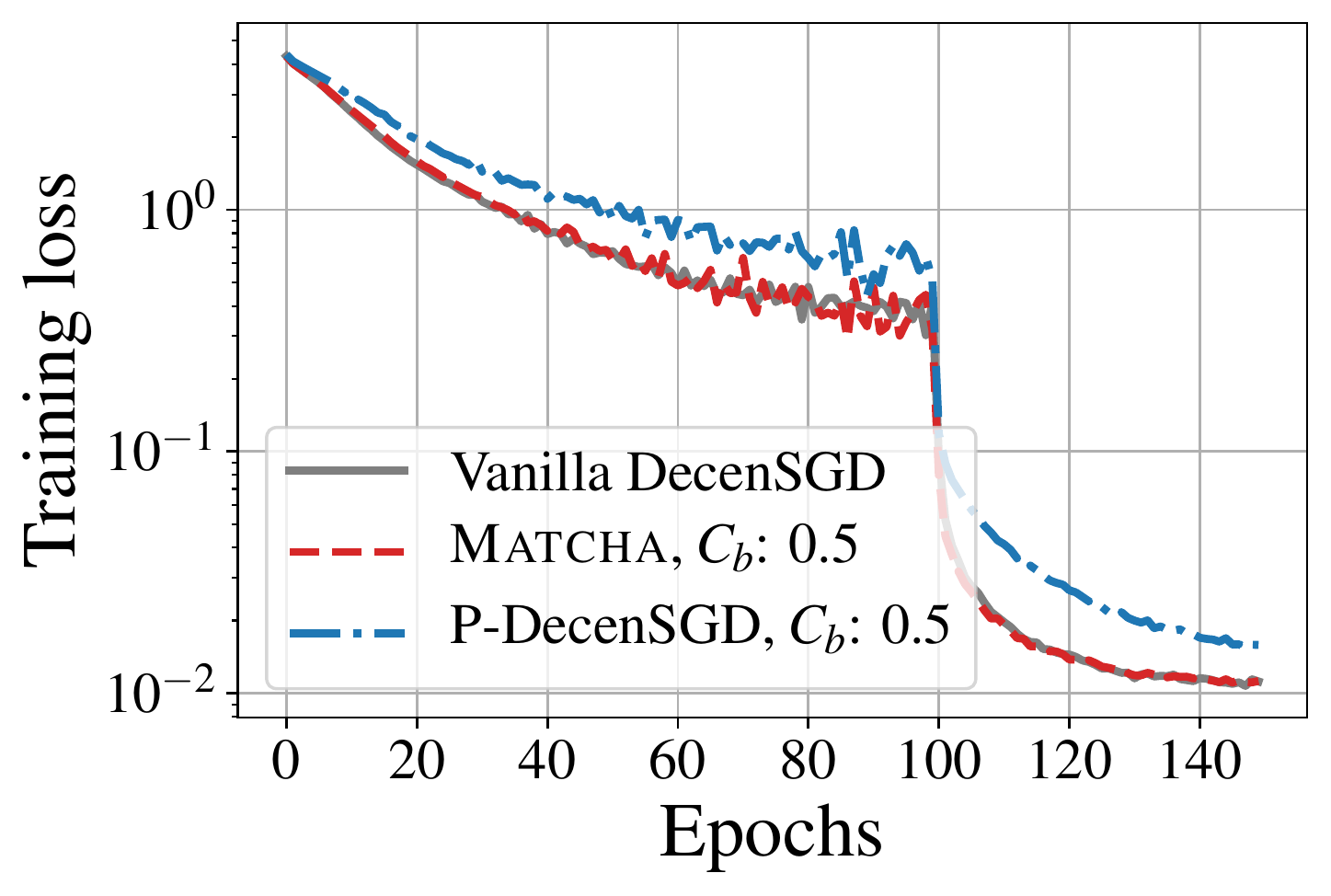}
    \caption{WideResNet on CIFAR-100.}
    \label{fig:c100_p0p1}
    \end{subfigure}%
    ~
    \begin{subfigure}{.3\textwidth}
    \centering
    \includegraphics[width=\textwidth]{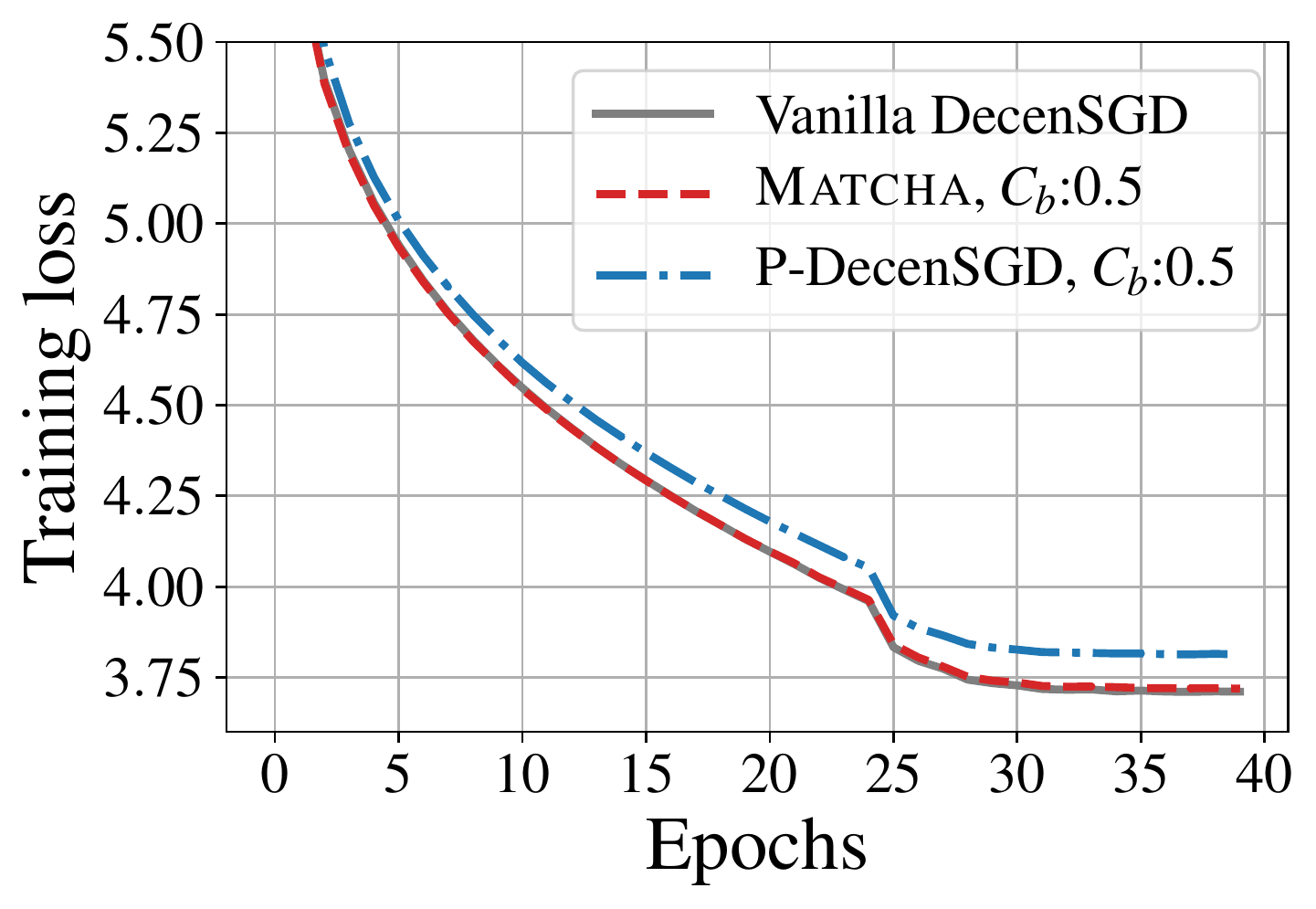}
    \caption{LSTM on Penn Treebank.}
    \label{fig:ptb_p0p1}
    \end{subfigure}%
    ~
    \begin{subfigure}{.3\textwidth}
    \centering
    \includegraphics[width=\textwidth]{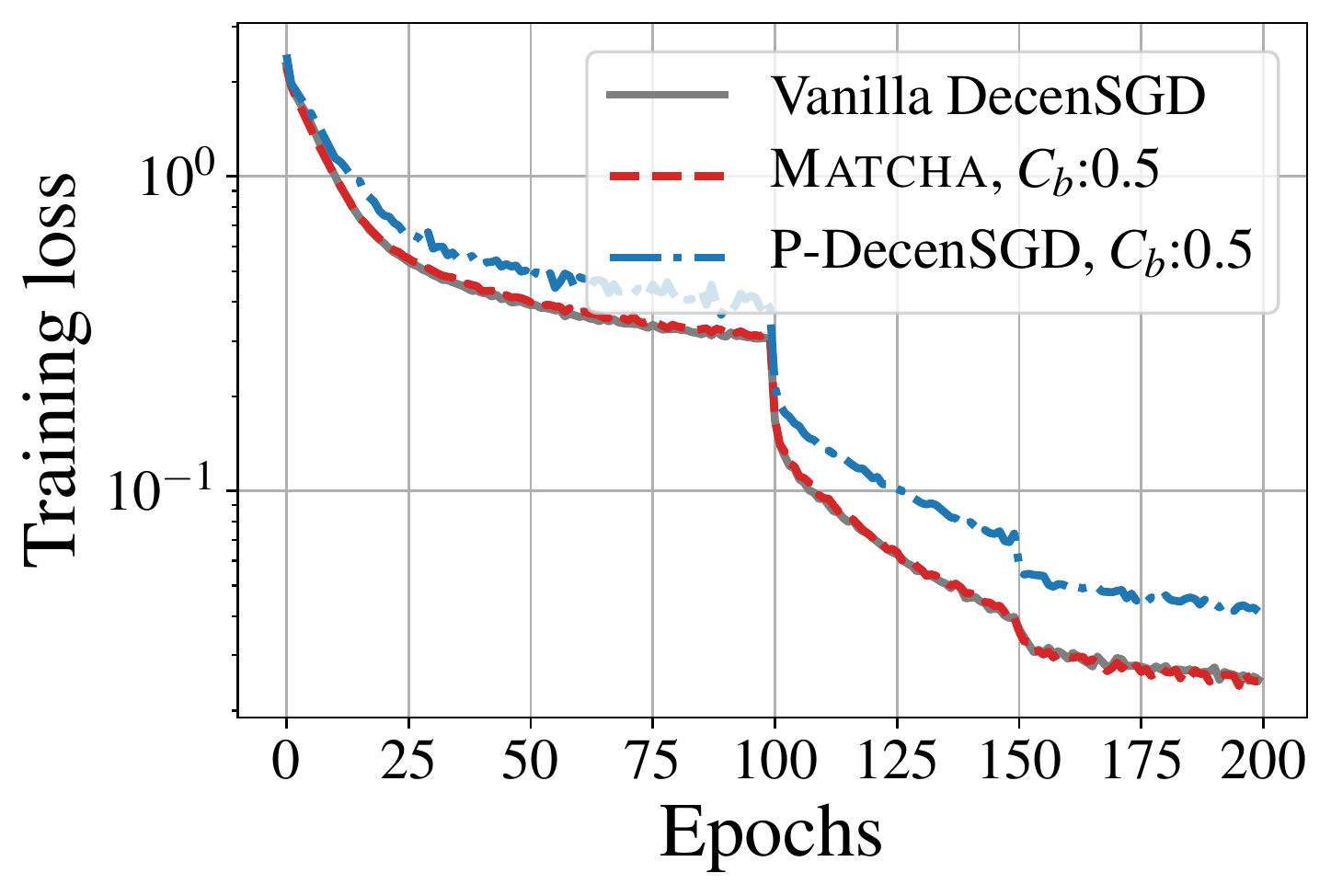}
    \caption{ResNet on CIFAR-10.}
    \label{fig:res_p0p1}
    \end{subfigure}%
    \caption{Comparision of {\algo} and P-DecenSGD. The base communication topology is given in Figure 1(a). While {\algo} has nearly identical error-convergence to vanilla DecenSGD, P-DecenSGD performs consistently worse in all tasks. Note that {\algo} and P-DecenSGD have the same average communication time per iteration.}
\end{figure}

\begin{figure}[!ht]
    \centering
    \begin{subfigure}{.3\textwidth}
    \centering
    \includegraphics[width=\textwidth]{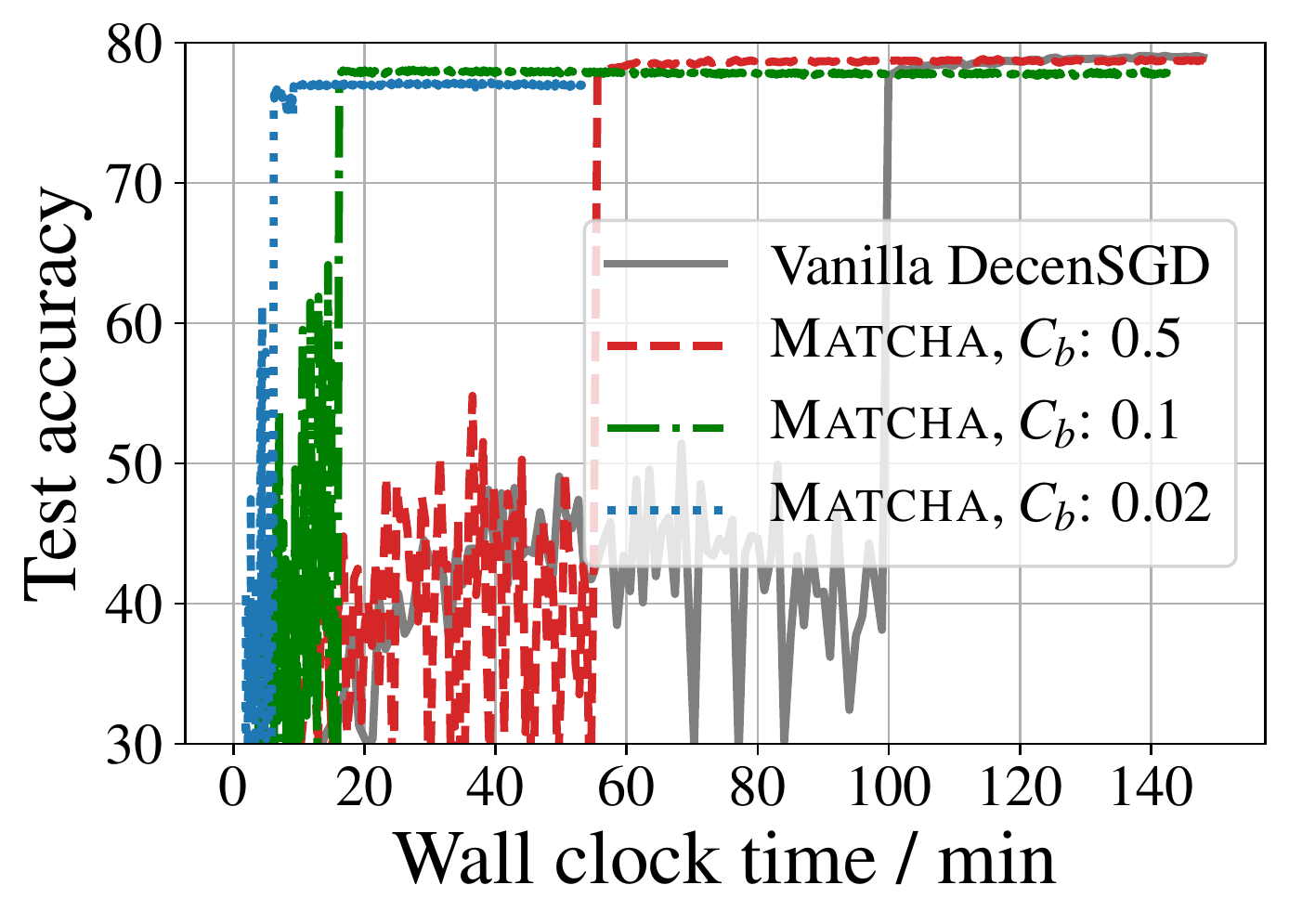}
    \caption{WideResNet on CIFAR-100.}
    \label{fig:acc_comp_c100}
    \end{subfigure}%
    ~
    \begin{subfigure}{.3\textwidth}
    \centering
    \includegraphics[width=\textwidth]{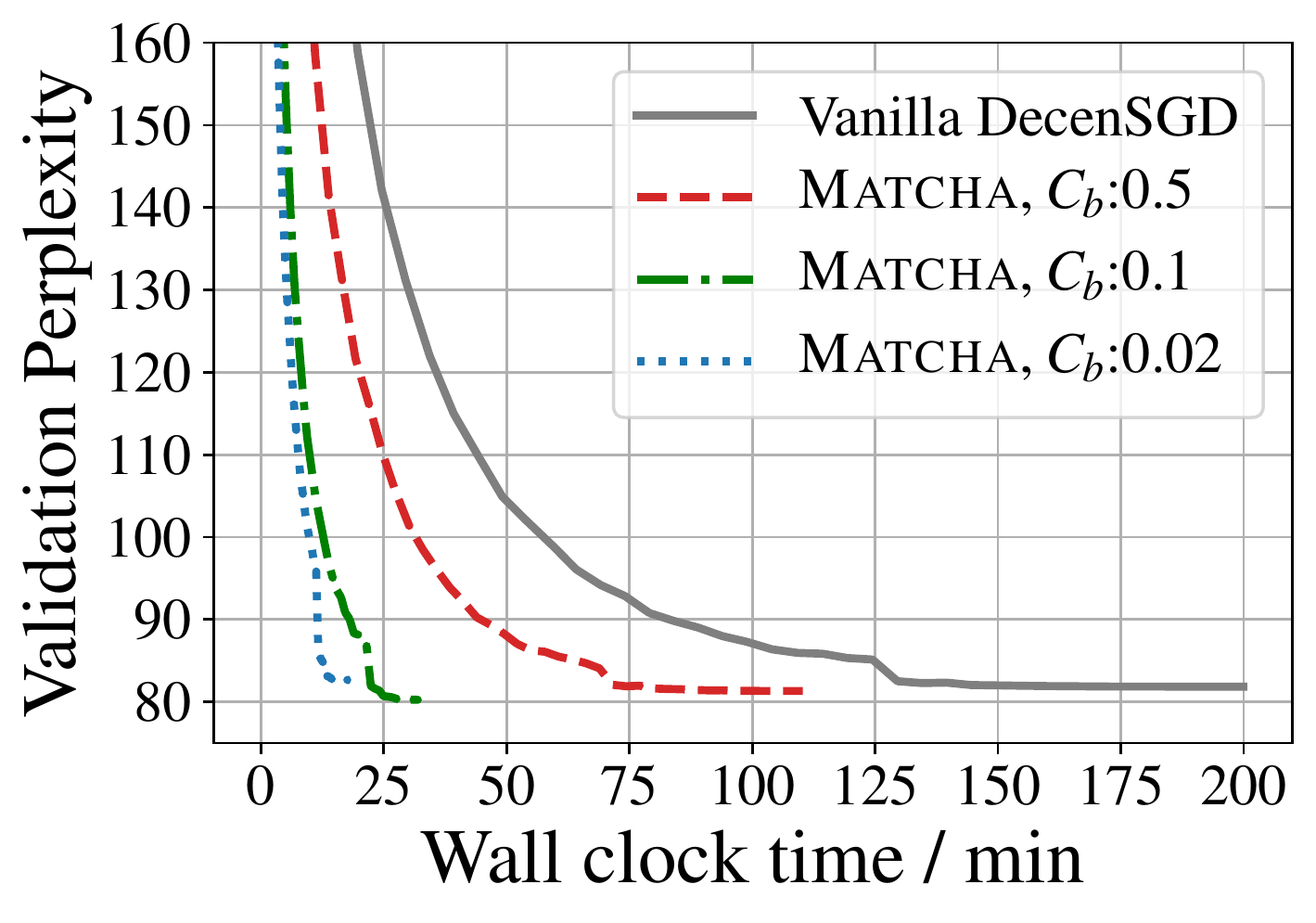}
    \caption{LSTM on Penn Treebank.}
    \label{fig:acc_comp_ptb}
    \end{subfigure}%
    ~
    \begin{subfigure}{.3\textwidth}
    \centering
    \includegraphics[width=\textwidth]{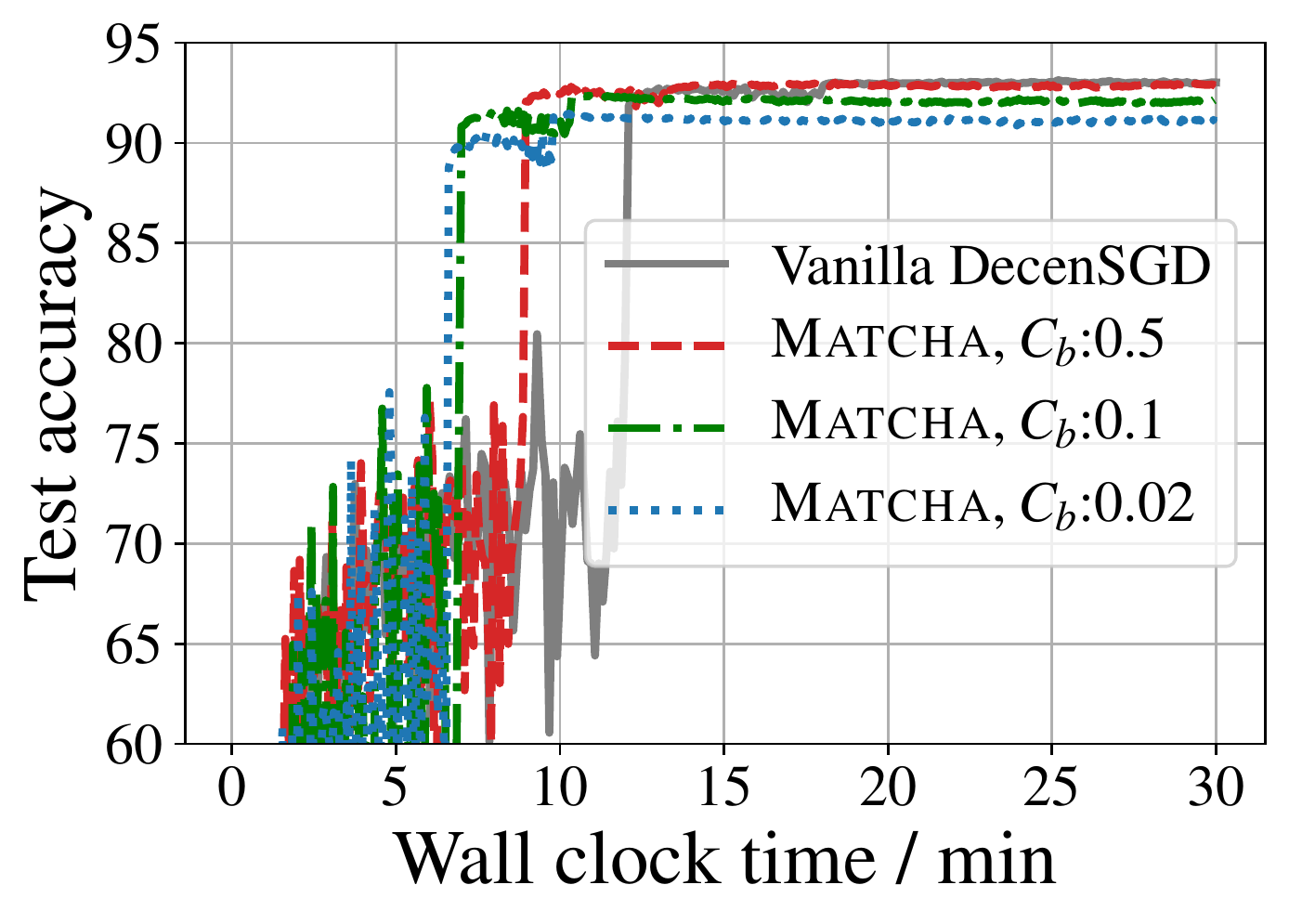}
    \caption{ResNet on CIFAR-10.}
    \label{fig:acc_comp_c10}
    \end{subfigure}%
    \caption{Test accuracy of {\algo} on different training tasks, corresponding to the training loss curves in Figure 4. The base communication topology is given in Figure 1(a). As predicted by the theoretical result Figure 1(c), {\algo} with $C_b=0.5$ consistently reaches the same test accuracy as {\dsgd}. When we further reduce the communication budget, the final test accuracy of {\algo} only slightly degrades.}
\end{figure}
\begin{figure}[ht!]
    \centering
    \begin{subfigure}{.3\textwidth}
    \centering
    \includegraphics[width=\textwidth]{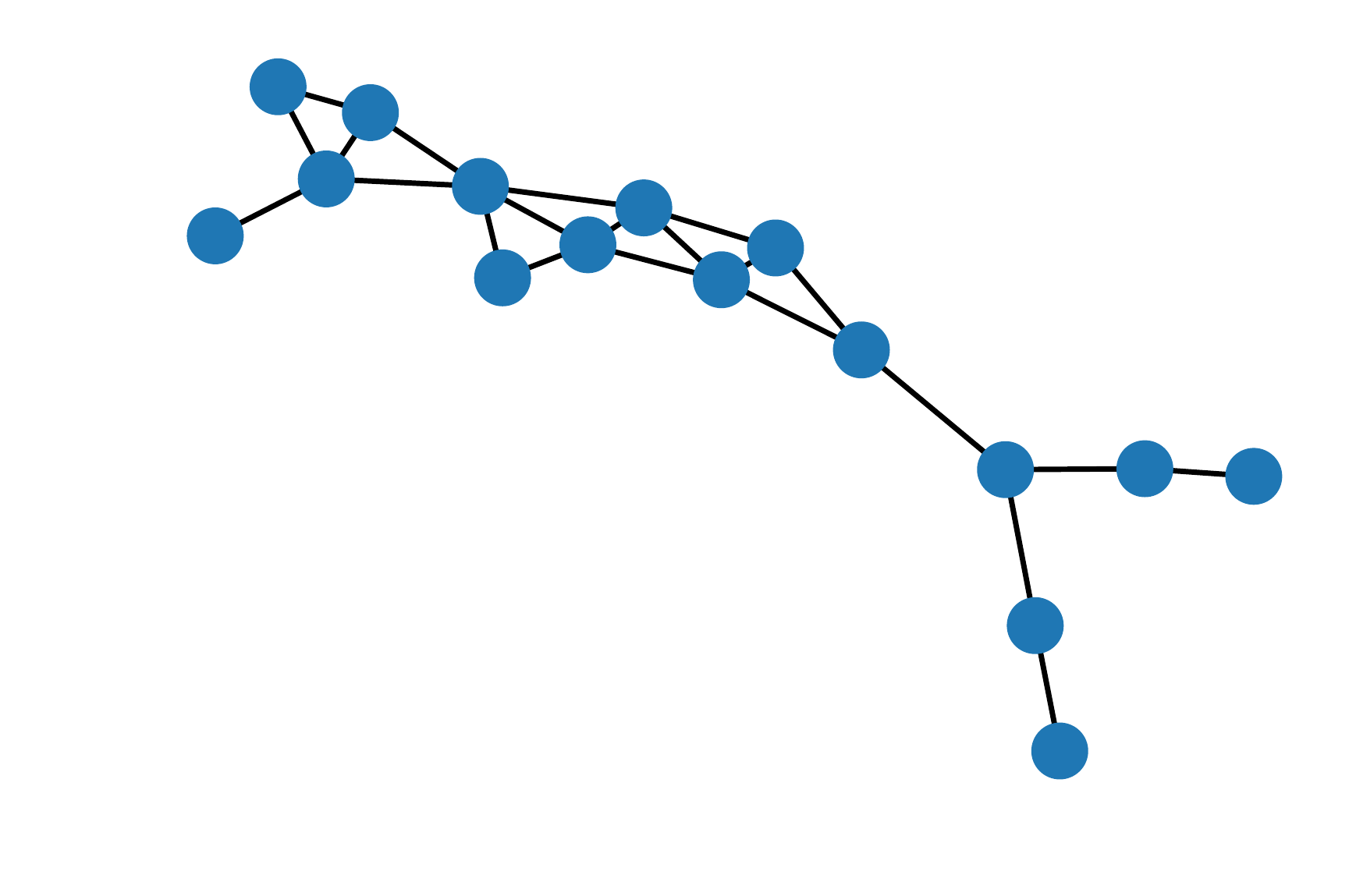}
    \caption{Maximal degree is $5$.}
    \label{fig:t1}
    \end{subfigure}%
    ~
    \begin{subfigure}{.3\textwidth}
    \centering
    \includegraphics[width=\textwidth]{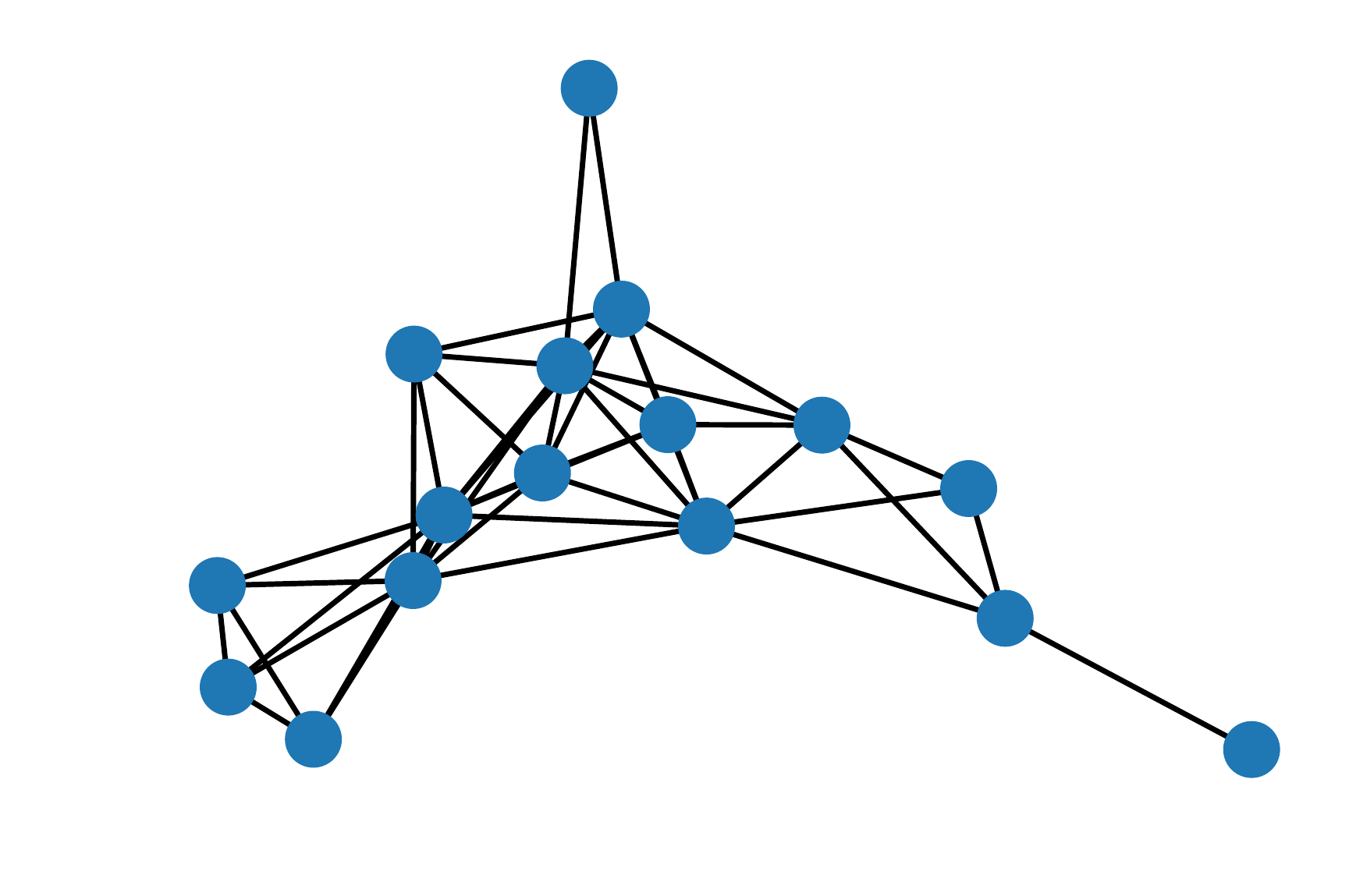}
    \caption{Maximal degree is $10$.}
    \label{fig:t2}
    \end{subfigure}%
    ~
    \begin{subfigure}{.3\textwidth}
    \centering
    \includegraphics[width=\textwidth]{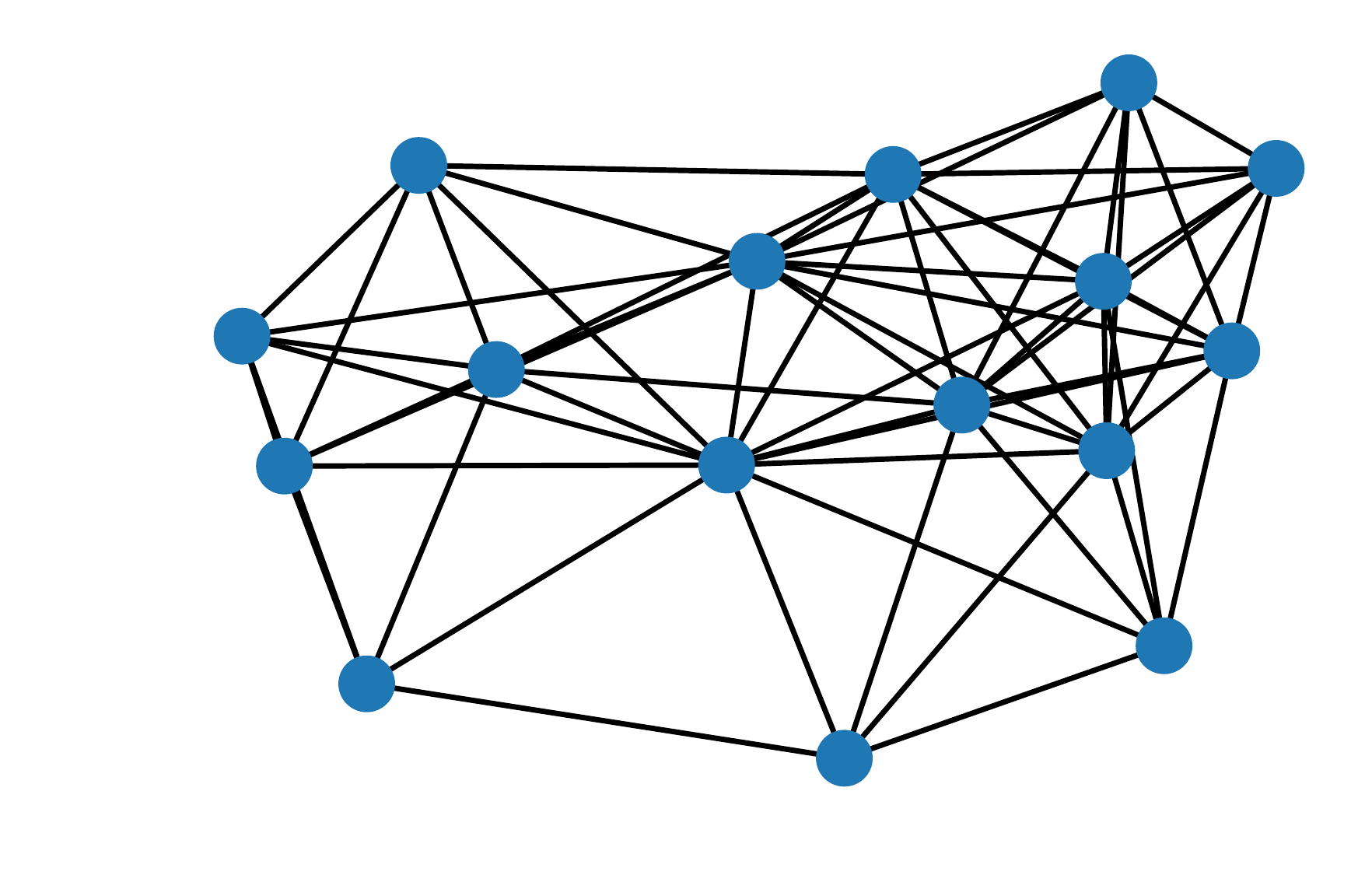}
    \caption{Maximal degree is $13$.}
    \label{fig:t3}
    \end{subfigure}%
    \caption{Different geometric topologies used in Figures 5.}
\end{figure}
\begin{figure}[ht!]
    \centering
    \begin{subfigure}{.3\textwidth}
    \centering
    \includegraphics[width=\textwidth]{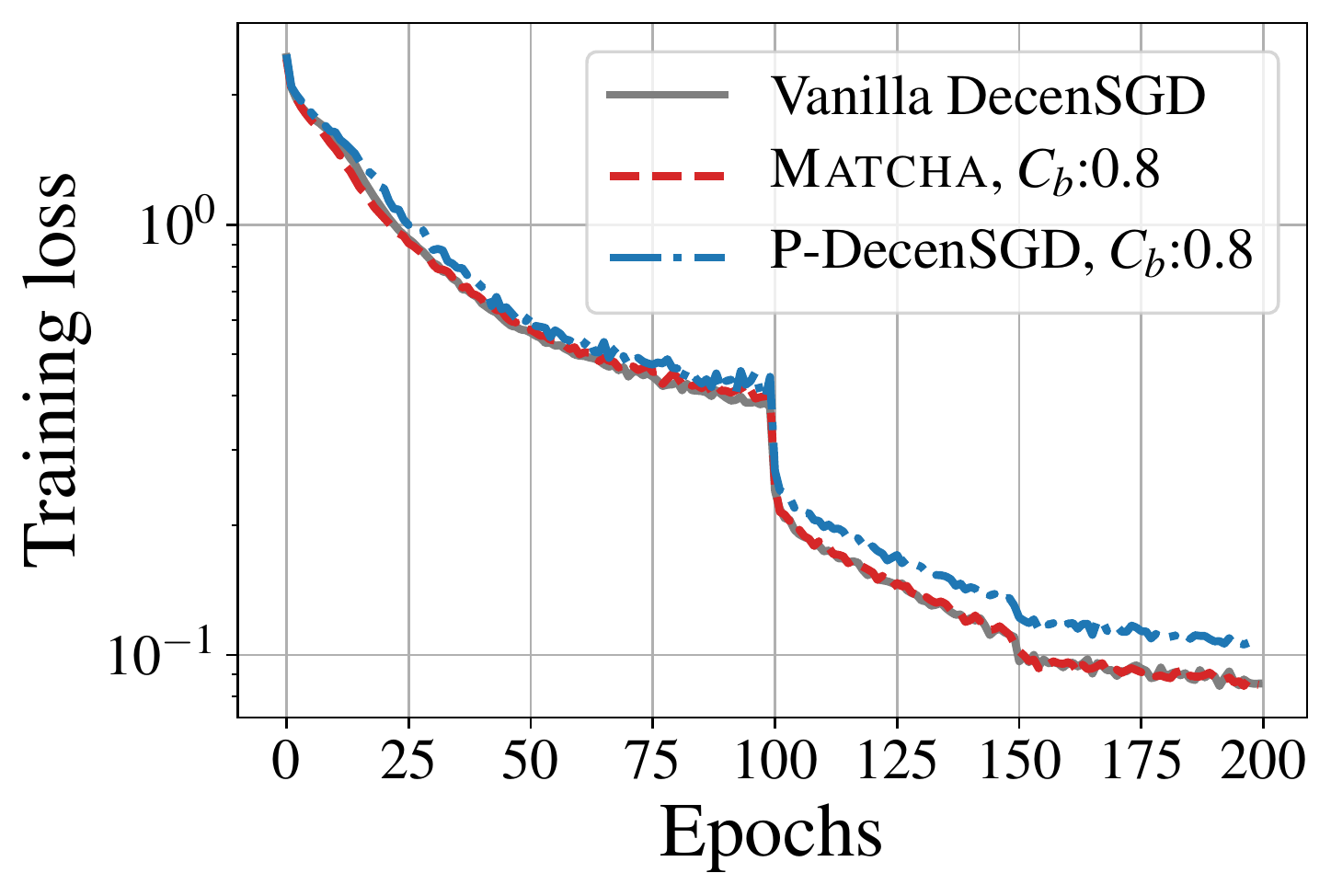}
    \caption{Maximal degree is $5$.}
    \end{subfigure}%
    ~
    \begin{subfigure}{.3\textwidth}
    \centering
    \includegraphics[width=\textwidth]{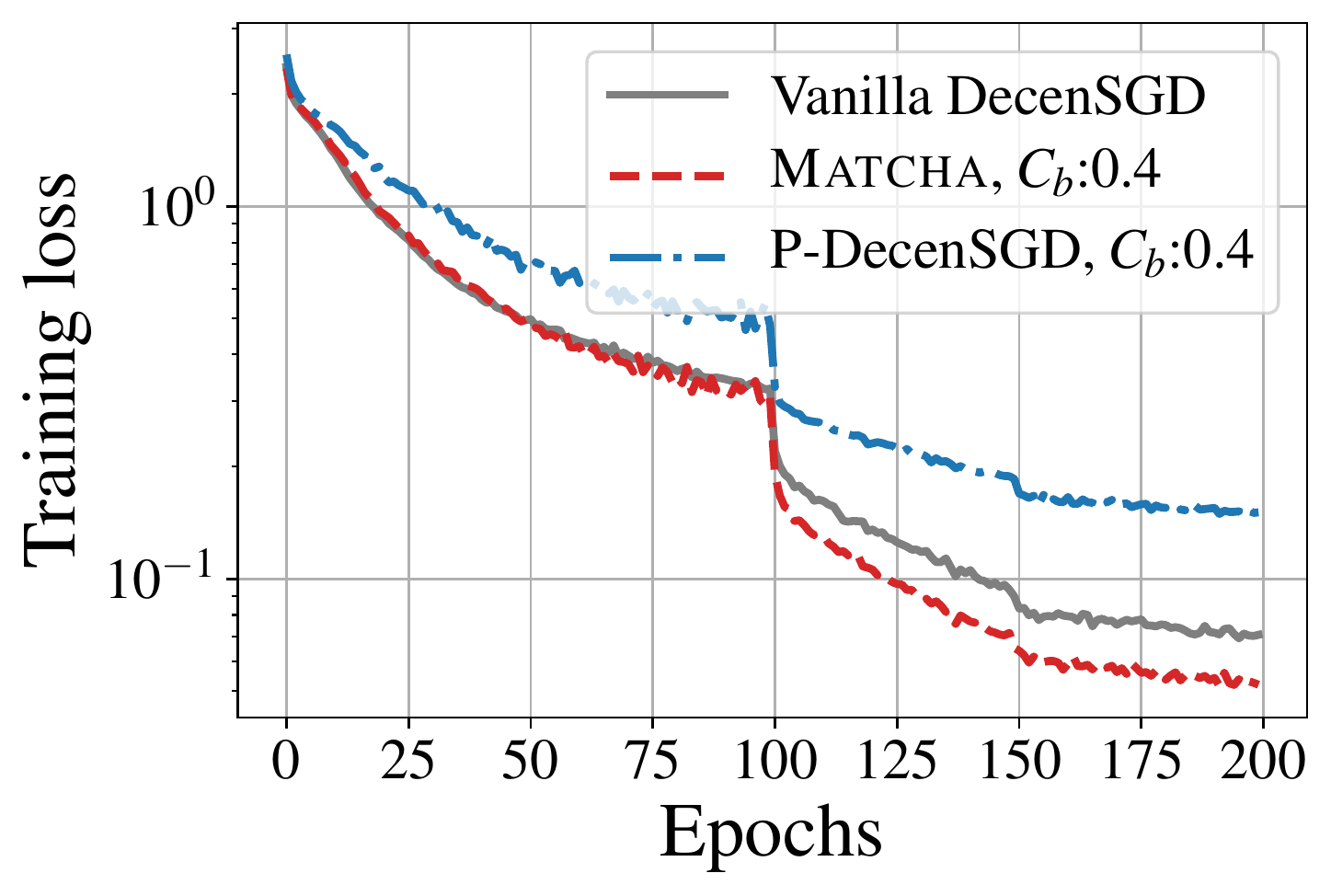}
    \caption{Maximal degree is $10$.}
    \end{subfigure}%
    ~
    \begin{subfigure}{.3\textwidth}
    \centering
    \includegraphics[width=\textwidth]{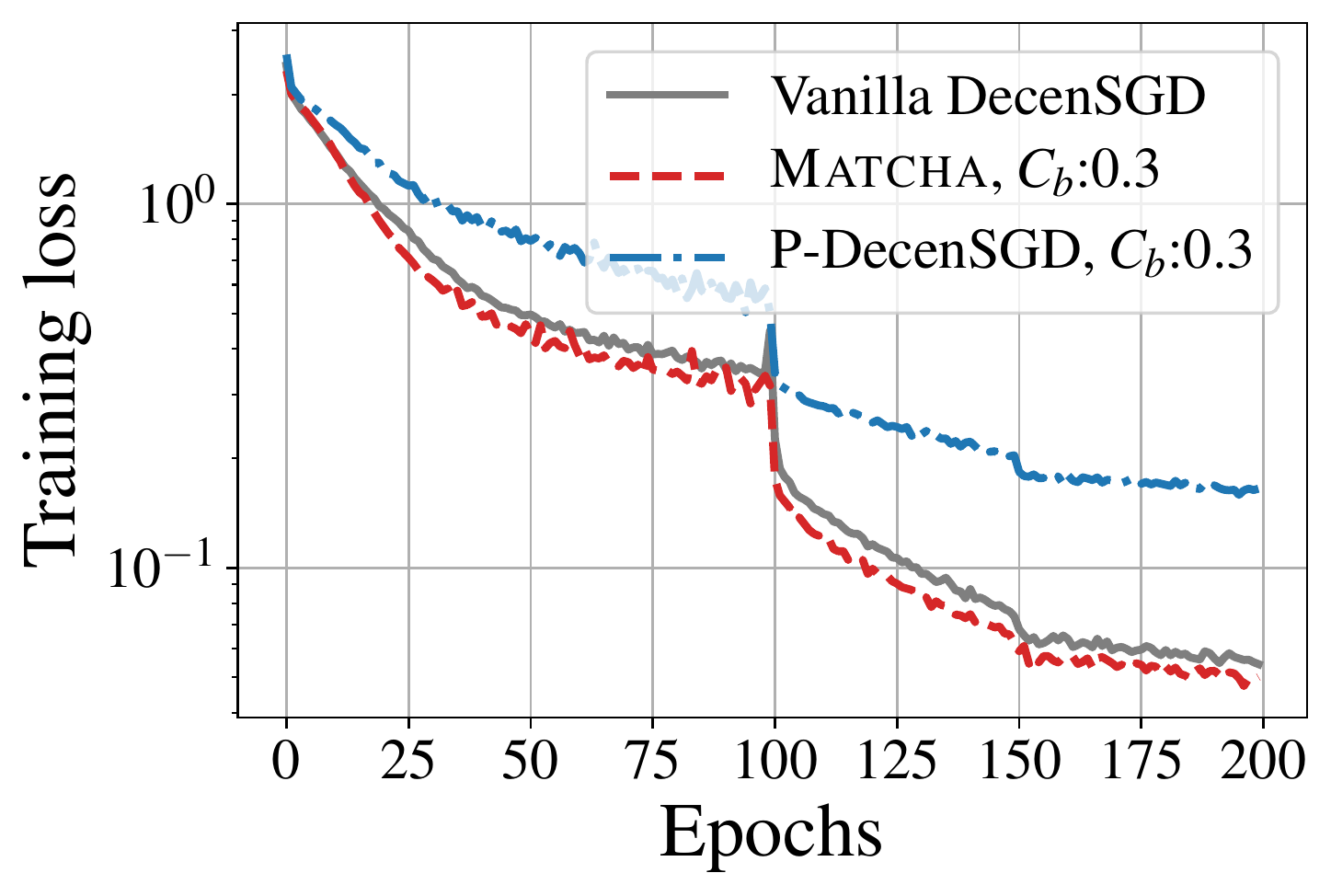}
    \caption{Maximal degree is $13$.}
    \end{subfigure}%
    \caption{Training loss versus epochs of {\algo} on different topologies with $16$ nodes. {\algo} can even have lower training loss than vanilla DecenSGD by setting a proper communication budget.}
\end{figure}

\section{Proof of Theorem 1}
The proof contains three parts: (1) we first show that the expected activated topology $\sum_{j=1}^{\nmatching} p_j \lap_j$  is connected, i.e., $\lambda_2(\sum_{j=1}^{\nmatching} p_j \lap_j) >0$; (2) then we will prove that if the expected topology is connected, then there must exist an $\alpha$ such that $\rho < 1$ for any arbitrary communication budget; (3) finally, we will show that $\alpha$ can be obtained via solving a semi-definite programming problem.

Recall that $\{p_j\}_{j=1}^{\nmatching}$ is the solution of convex optimization problem (5). Let $p_0 = C_b$, then we have
\begin{align}
    \lambda_2(\sum_{j=1}^{\nmatching} p_j \lap_j) \geq \lambda_2(p_0 \sum_{j=1}^{\nmatching} \lap_j) = p_0 \lambda_2(\sum_{j=1}^{\nmatching} \lap_j) >0.
\end{align}
The last inequality comes from the fact: the base communication topology is connected, i.e., $\lambda_2(\sum_{j=1}^{\nmatching} \lap_j) >0$. Here we complete the first part of the proof. Then, recall the definition of $\rho$ and $\mixmat^{(k)}$, we obtain
\begin{align}
    \opnorm{\Exs\brackets{\mixmat^{(k)\top} \mixmat^{(k)}} - \J}
    &= \opnorm{\Exs\brackets{\parenth{\I - \alpha \lap^{(k)}}\tp\parenth{\I - \alpha \lap^{(k)}}}-\J}\\
    &= \opnorm{\I - 2\alpha \Exs\brackets{\lap^{(k)}}+\alpha^2 \Exs\brackets{\lap^{(k)\top}\lap^{(k)}}-\J} \label{eqn:sp_bnd1}
\end{align}
where $\lap^{(k)} = \sum_{j=1}^{\nmatching} \actrv_j^{(k)} \lap_j$. Since $\actrv_j^{(k)}$'s are i.i.d. across all subgraphs and iterations,
\begin{align}
    \Exs\brackets{\lap^{(k)}}
    =& \sum_{j=1}^{\nmatching} p_j \lap_j \label{eqn:mean_lap}\\
    \Exs\brackets{\lap^{(k)\top}\lap^{(k)}}
    =& \sum_{j=1}^{\nmatching} p_j^2\lap_j^2 + \sum_{j=1}^{\nmatching}\sum_{t=1,t\neq j}^{\nmatching} p_j p_t\lap_j\tp \lap_t + \sum_{j=1}^{\nmatching} p_j(1-p_j)\lap_j^2 \\
    =& \parenth{\sum_{j=1}^{\nmatching} p_j\lap_j}^2 + \sum_{j=1}^{\nmatching} p_j(1-p_j)\lap_j^2 \\
    =& \parenth{\sum_{j=1}^{\nmatching} p_j\lap_j}^2 + 2\sum_{j=1}^{\nmatching} p_j(1-p_j)\lap_j. \label{eqn:var_lap}
\end{align}
Plugging \Cref{eqn:mean_lap,eqn:var_lap} back into \Cref{eqn:sp_bnd1}, we get
\begin{align}
    \opnorm{\Exs\brackets{\mixmat^{(k)\top} \mixmat^{(k)}} - \J}
    &= \opnorm{\parenth{\I - \alpha \sum_{j=1}^{\nmatching}p_j\lap_j}^2 + 2\alpha^2\sum_{j=1}^{\nmatching} p_j(1-p_j)\lap_j - \J} \\
    &\leq \opnorm{\parenth{\I - \alpha \sum_{j=1}^{\nmatching}p_j\lap_j}^2 - \J} + 2\alpha^2 \opnorm{\sum_{j=1}^{\nmatching} p_j(1-p_j)\lap_j} \\
    &= \max\{|1-\alpha \lambda_2|^2,|1-\alpha \lambda_m|^2\} + 2\alpha^2 \zeta
\end{align}
where $\lambda_i$ denotes the $i$-th smallest eigenvalue of matrix $\sum_{j=1}^{\nmatching} p_j\lap_j$ and $\zeta \geq 0$ denotes the spectral norm of matrix $\sum_{j=1}^{\nmatching} p_j(1-p_j)\lap_j$. Suppose $h_\lambda(\alpha) = (1-\alpha \lambda)^2 + 2\alpha^2 \zeta$. Then, we have
\begin{align}
    \frac{\partial h}{\partial \alpha} =& -2\lambda(1-\alpha\lambda) + 4\alpha\zeta, \\
    \frac{\partial^2 h}{\partial \alpha^2} =& 2\lambda^2 + 4\zeta >0.
\end{align}
Therefore, $h_\lambda(\alpha)$ is a convex funtion. By setting its derivative to zero, we can get the minimal value:
\begin{align}
    \alpha^* &= \frac{\lambda}{\lambda^2 + 2\zeta}, \\
    h_\lambda(\alpha^*) &= \frac{4\zeta^2}{(\lambda^2+2\zeta)^2} + \frac{2\lambda^2\zeta}{(\lambda^2+2\zeta)^2} = \frac{2\zeta}{\lambda^2 + 2\zeta}.
\end{align}
We already prove that $\lambda_m \geq \lambda_2 > 0$. Therefore, $\alpha^* > 0$ and $h_\lambda(\alpha^*) <1$.
Furthermore, note that $h_\lambda(0) = 1$ and $h_\lambda(\alpha)$ is a quadratic function. We can conclude that when $\alpha \in (0, 2\alpha^*)$, $h_\lambda(\alpha^*) \leq h_\lambda(\alpha) < 1$. Thus, when $\alpha \in (0, \min\{\frac{2\lambda_2}{\lambda_2^2+2\zeta},\frac{2\lambda_\nworkers}{\lambda_\nworkers^2+2\zeta}\})$, we have
\begin{align}
    \opnorm{\Exs\brackets{\mixmat^{(k)\top} \mixmat^{(k)}} - \J}
    \leq& \max\{h_{\lambda_2}(\alpha), h_{\lambda_\nworkers}(\alpha)\} < 1.
\end{align}

\subsection{Formulation of the Semi-Definite Programming Problem}
It has been shown that the spectral norm $\rho$ can be expanded as
\begin{align}
    \opnorm{\Exs\brackets{\mixmat^{(k)\top} \mixmat^{(k)}} - \J}
    &= \opnorm{\I - 2\alpha \sum_{j=1}^{\nmatching}p_j\lap_j + \alpha^2 \parenth{\sum_{j=1}^{\nmatching}p_j\lap_j}^2 + 2\alpha^2\sum_{j=1}^{\nmatching} p_j(1-p_j)\lap_j - \J} \\
    &= \opnorm{\I - 2\alpha \overline{\lap} + \alpha^2 \overline{\lap}^2 + 2\alpha^2\widetilde{\lap} - \J}
\end{align}
where $\overline{\lap} = \sum_{j=1}^\nmatching \lap_j$ and $\widetilde{\lap}=\sum_{j=1}^\nmatching p_j(1-p_j)\lap_j$. Our goal is to find a value of $\alpha$ that minize the spectral norm:
\begin{align}
    \min_\alpha \quad \opnorm{\I - 2\alpha \overline{\lap} + \alpha^2 [\overline{\lap}^2 + 2\widetilde{\lap}] - \J}
\end{align}
which is equivalent to 
\begin{align}\label{eqn:prev_opt_a}
\begin{split}
    \min_{\rho,\alpha} \quad &\rho \\
    \text{subject to} \quad & \I - 2\alpha\overline{\lap} + \alpha^2[\overline{\lap}^2+ 2\widetilde{\lap}] - \frac{1}{\nworkers}\one\one\tp \preceq \rho\I.
\end{split}
\end{align}
However, directly solving \Cref{eqn:prev_opt_a} is N-P hard as it has bilinear matrix inequality constraint. We relax the above optimization problem by introducing an auxiliary variable $\beta$ as follows:
\begin{align}\label{eqn:opt_a_1}
\begin{split}
    \min_{\rho,\alpha,\beta} \quad &\rho \\
    \text{subject to} \quad & \alpha^2 - \beta \leq 0,\ \I - 2\alpha\overline{\lap} + \beta[\overline{\lap}^2+ 2\widetilde{\lap}] - \frac{1}{\nworkers}\one\one\tp \preceq \rho\I.
\end{split}
\end{align}
Now the constraints become linear matrix inequality constraints and \Cref{eqn:opt_a_1} is the standard form of semi-definite programming. However, we need to further show that the solution of \Cref{eqn:opt_a_1} is same as \Cref{eqn:prev_opt_a}. We will prove this by contradiction. Suppose $\alpha_+, \beta_+, \rho_+$ are the solution of problem  \Cref{eqn:opt_a_1} and they satisfy $\alpha_+^2 < \beta_+$. Without loss of generality, we can simply assume $\beta_+ = \alpha_+^2 + c$, where c is a positive constant. Then, we have
\begin{align}
	\I - 2\alpha_+\overline{\lap} + (\alpha_+^2+c)[\overline{\lap}^2+ 2\widetilde{\lap}] - \frac{1}{\nworkers}\one\one\tp \preceq \rho_+\I.
\end{align}
Furthermore, according to the definitions of $\overline{\lap}$ and $\widetilde{\lap}$, both of these matrix are positive semi-definite and have positive largest eigenvalues. As a result, we can obtain
\begin{align}
	\I - 2\alpha_+\overline{\lap} + \alpha_+^2[\overline{\lap}^2+ 2\widetilde{\lap}] - \frac{1}{\nworkers}\one\one\tp \prec \I - 2\alpha_+\overline{\lap} + (\alpha_+^2+c)[\overline{\lap}^2+ 2\widetilde{\lap}] - \frac{1}{\nworkers}\one\one\tp \preceq \rho_+\I.
\end{align}
That is to say, there must exist $\rho_*$ such that 
\begin{align}
	\I - 2\alpha_+\overline{\lap} + \alpha_+^2[\overline{\lap}^2+ 2\widetilde{\lap}] - \frac{1}{\nworkers}\one\one\tp \preceq \rho_* \I \prec \rho_+\I.
\end{align}
So $\rho_+$ is not the optimal solution. Our assumptions cannot hold. The solutions of  \Cref{eqn:opt_a_1} must satisfy $\alpha^2=\beta$.

\subsection{Discussions on the Upper Bound of Spectral Norm}
Here, we are going to show that optimizing $p_i$'s via (5) is equivalent to minimizing an upper bound of the spectral norm $\rho$. Recall that 
\begin{align}
    \opnorm{\Exs\brackets{\mixmat^{(k)\top} \mixmat^{(k)}} - \J}
    =& \opnorm{\I - 2\alpha \overline{\lap} + \alpha^2 \overline{\lap}^2 + 2\alpha^2\widetilde{\lap} - \J} \\
    \leq& \opnorm{\I - 2\alpha \overline{\lap} - \J} + \alpha^2\opnorm{ \overline{\lap}^2} + 2\alpha^2\opnorm{\widetilde{\lap}} \\
    =& (1-2\alpha \lambda_2(\overline{\lap})) + \alpha^2\opnorm{ \overline{\lap}^2} + 2\alpha^2\opnorm{\widetilde{\lap}}.
\end{align}
According to the definition of $\overline{\lap}$ and $\widetilde{\lap}$ and repeatedly using Cauchy-Schwartz inequality and triangle inequality, we have
\begin{align}
    \opnorm{\overline{\lap}^2}
    &\leq \opnorm{\overline{\lap}}^2 
    \leq \opnorm{\sum_{j=1}^\nmatching p_j \lap_j}^2
    \leq (\sum_{j=1}^\nmatching p_j \opnorm{\lap_j})^2 
    \leq 4(\sum_{j=1}^\nmatching p_j)^2
    \leq 4 C_b^2, \\
    \opnorm{\widetilde{\lap}}
    &= \opnorm{\sum_{j=1}^\nmatching p_j(1-p_j) \lap_j}
    \leq \sum_{j=1}^\nmatching p_j(1-p_j) \opnorm{\lap_j} \leq 2\sum_{j=1}^\nmatching p_j(1-p_j) \leq 2\sum_{j=1}^\nmatching p_j \leq 2C_b.
\end{align}
As a result, one can get
\begin{align}
    \opnorm{\Exs\brackets{\mixmat^{(k)\top} \mixmat^{(k)}} - \J}
    \leq 1 - 2\alpha \lambda_2(\overline{\lap}) + 4\alpha^2C_b^2 + 4\alpha^2C_b. \label{eqn:rho_ub}
\end{align}
Therefore, maximizing $\lambda_2(\overline{\lap})$ is equivalent to minimizing the RHS in \eqref{eqn:rho_ub} --- an upper bound of the spectral norm $\rho$.

\section{Proofs of Theorem 2 and Corollary 1}
\subsection{Preliminaries}
In the proof, we will use the following matrix forms:
\begin{align}
    \X^{(k)} =& \brackets{\x_1^{(k)}, \x_2^{(k)},\dots,\x_m^{(k)}},\\
    \Sg^{(k)} =& \brackets{\sg_1(\x_1^{(k)}),\sg_2(\x_2^{(k)}),\dots, \sg_m(\x_m^{(k)})},\\
    \Tg^{(k)} =& \brackets{\tg_1(\x_1^{(k)}),\tg_2(\x_2^{(k)}),\dots,\tg_m(\x_m^{(k)})}.
\end{align}
Recall the assumptions we make:
\begin{align}
    \vecnorm{\tg_i(\x) - \tg_i(\mathbf{y})} \leq L \vecnorm{\x - \mathbf{y}}, \\
    \Exs\brackets{\sg_i(\x)|\x} = \tg_i(\x), \\
    \Exs\brackets{\vecnorm{\sg_i(\x) - \tg_i(\x)}^2|\x} \leq \vbnd,\\
    \frac{1}{\nworkers}\sum_{i=1}^\nworkers \vecnorm{\tg_i(\x)-\tg(\x)}^2\leq \zeta^2.
\end{align}
The matrix form update rule can be written as:
\begin{align}
    \X^{(k+1)} = \parenth{\X^{(k)} - \lr \Sg^{(k)}}\mixmat^{(k)}.
\end{align}
Multiplying a vector $\one/\nworkers$ on both sides, we have
\begin{align}
    \avgx^{(k+1)} = \avgx^{(k)} - \frac{\lr}{\nworkers}\Sg^{(k)}\one.
\end{align}

\subsection{Lemmas}
\begin{lem}\label{lem:contract}
Let $\{\mixmat^{(k)}\}_{k=1}^{\infty}$ be an i.i.d. symmetric and doubly stochastic matrices sequence. The size of each matrix is $\nworkers \times \nworkers$. Then, for any matrix $\mathbf{B} \in \mathbb{R}^{d\times \nworkers}$,
\begin{align}
    \Exs\brackets{\fronorm{\mathbf{B}\parenth{\prod_{l=1}^n \mixmat^{(l)} - \J}}^2} \leq \rho^n \fronorm{\mathbf{B}}^2
\end{align}
where $\rho := \sigma_\text{max}(\Exs[\mixmat^{(k)\top}\mixmat^{(k)}]-\J)$ = \opnorm{\Exs[\mixmat^{(k)\top}\mixmat^{(k)}]-\J}.
\end{lem}
\begin{proof}
For the ease of writing, let us define $\mathbf{A}_{q,n} := \prod_{l=q}^n \mixmat^{(l)} - \J$ and use $\mathbf{b}_i\tp$ to denote the $i$-th row vector of $\mathbf{B}$. Since for all $k \in \mathbb{N}$, we have $\mixmat^{(k)\top} = \mixmat^{(k)}$ and $\mixmat^{(k)}\J = \J\mixmat^{(k)} = \J$. Thus, one can obtain
\begin{align}
    \mathbf{A}_{1,n} = \prod_{k=1}^n \parenth{\mixmat^{(k)}-\J} = \mathbf{A}_{1,n-1}\parenth{\mixmat^{(n)}-\J}.
\end{align}
Then, taking the expectation with respect to $\mixmat^{(n)}$,
\begin{align}
    \Exs_{\mixmat^{(n)}}\brackets{\fronorm{\mathbf{B}\mathbf{A}_{1,n}}^2}
    =& \sum_{i=1}^d \Exs_{\mixmat^{(n)}}\brackets{\vecnorm{\mathbf{b}_i\tp \mathbf{A}_{1,n}}^2} \\
    =& \sum_{i=1}^d \Exs_{\mixmat^{(n)}}\brackets{\mathbf{b}_i\tp \mathbf{A}_{1,n-1}(\mixmat^{(n)\top}\mixmat^{(n)} - \J)\mathbf{A}_{1,n-1}\tp\mathbf{b}_i} \\
    =& \sum_{i=1}^d \mathbf{b}_i\tp \mathbf{A}_{1,n-1}\Exs_{\mixmat^{(n)}}\brackets{(\mixmat^{(n)\top}\mixmat^{(n)} - \J)}\mathbf{A}_{1,n-1}\tp\mathbf{b}_i.
\end{align}
Let $\mathbf{C} = \Exs_{\mixmat^{(n)}}\brackets{(\mixmat^{(n)\top}\mixmat^{(n)}-\J)}$ and $\mathbf{v}_i = \mathbf{A}_{1,n-1}\tp\mathbf{b}_i$, then
\begin{align}
    \Exs_{\mixmat^{(n)}}\brackets{\fronorm{\mathbf{B}\mathbf{A}_{1,n}}^2}
    =& \sum_{i=1}^d \mathbf{v}_i\tp\mathbf{C}\mathbf{v}_i \\
    \leq& \sigma_\text{max}(\mathbf{C}) \sum_{i=1}^d \mathbf{v}_i\tp\mathbf{v}_i \\
    =& \rho \fronorm{\mathbf{B}\mathbf{A}_{1,n-1}}^2.
\end{align}
Repeat the following procedure, since $\mixmat^{(k)}$'s are i.i.d. matrices, we have
\begin{align}
    \Exs_{\mixmat^{(1)}}\dots \Exs_{\mixmat^{(n-1)}}\Exs_{\mixmat^{(n)}} \brackets{\fronorm{\mathbf{B}\mathbf{A}_{1,n}}^2}
    \leq \rho^n \fronorm{\mathbf{B}}^2.
\end{align}
Here, we complete the proof.
\end{proof}

\subsection{Proof of Theorem 1}
Since the objective function $F(\x)$ is Liptchitz smooth, it means that
\begin{align}
    F(\avgx^{(k+1)}) - F(\avgx^{(k)})
    \leq \inprod{\tg(\avgx^{(k)})}{\avgx^{(k+1)}-\avgx^{(k)}} + \frac{\lip}{2}\vecnorm{\avgx^{(k+1)}-\avgx^{(k)}}^2. 
\end{align}
Plugging into the update rule $\avgx^{(k+1)} = \avgx^{(k)} - \lr \Sg^{(k)}\one/\nworkers$, we have
\begin{align}
    F(\avgx^{(k+1)}) - F(\avgx^{(k)})
    \leq -\lr\inprod{\tg(\avgx^{(k)})}{\frac{\Sg^{(k)}\one}{\nworkers}} + \frac{\lr^2\lip}{2}\vecnorm{\frac{\Sg^{(k)}\one}{\nworkers}}^2.
\end{align}
Then, taking the expectation with respect to random mini-batches at $k$-th iteration,
\begin{align}
    \Exs_k \brackets{F(\avgx^{(k+1)})} - F(\avgx^{(k)})
    \leq& -\lr\inprod{\tg(\avgx^{(k)})}{\frac{\Tg^{(k)}\one}{\nworkers}} + \frac{\lr^2\lip}{2}\Exs_k\brackets{\vecnorm{\frac{\Sg^{(k)}\one}{\nworkers}}^2}.
    \label{eqn:init_bnd}
\end{align}
For the first term in \eqref{eqn:init_bnd}, since $2\inprod{a}{b} = \vecnorm{a}^2+\vecnorm{b}^2-\vecnorm{a-b}^2$, we have
\begin{align}
    \inprod{\tg(\avgx^{(k)})}{\frac{\Tg^{(k)}\one}{\nworkers}}
    =&  \inprod{\tg(\avgx^{(k)})}{\frac{1}{\nworkers}\sum_{i=1}^\nworkers\tg_i(\x_i^{(k)})} \\
    =& \frac{1}{2} \brackets{\vecnorm{\tg(\avgx^{(k)})}^2 + \vecnorm{\frac{1}{\nworkers}\sum_{i=1}^\nworkers\tg_i(\x_i^{(k)})}^2 - \vecnorm{\tg(\avgx^{(k)}) - \frac{1}{\nworkers}\sum_{i=1}^\nworkers\tg_i(\x_i^{(k)})}^2 }
    \label{eqn:init_bnd2}
\end{align}
Recall that $\tg(\avgx^{(k)}) = \frac{1}{\nworkers}\sum_{i=1}^\nworkers \tg_i(\avgx)$,
\begin{align}
    \vecnorm{\tg(\avgx^{(k)}) - \frac{1}{\nworkers}\sum_{i=1}^\nworkers\tg_i(\x_i^{(k)})}^2
    =& \vecnorm{\frac{1}{\nworkers}\sum_{i=1}^\nworkers\brackets{\tg_i(\avgx^{(k)}) - \tg_i(\x_i^{(k)})}}^2 \\
    \stackrel{\text{Jensen's Inequality}}{\leq}& \frac{1}{\nworkers}\sum_{i=1}^\nworkers \vecnorm{\tg_i(\avgx^{(k)}) - \tg_i(\x_i^{(k)})}^2 \\
    \leq& \frac{\lip^2}{\nworkers}\sum_{i=1}^\nworkers\vecnorm{\avgx^{(k)}-\x_i^{(k)}}^2
    \label{eqn:bias2}
\end{align}
where the last inequality follows the Lipschitz smooth assumption. Then, plugging \Cref{eqn:bias2} into \eqref{eqn:init_bnd2}, we obtain
\begin{align}
    \inprod{\tg(\avgx^{(k)})}{\frac{\Tg^{(k)}\one}{\nworkers}}
    \geq& \frac{1}{2} \vecnorm{\tg(\avgx^{(k)})}^2 +\frac{1}{2} \vecnorm{\frac{\Tg^{(k)}\one}{\nworkers}}^2 - \frac{\lip^2}{2\nworkers}\fronorm{\X^{(k)}(\I-\J)}^2.
    \label{eqn:init_part1}
\end{align}
Next, for the second part in \eqref{eqn:init_bnd},
\begin{align}
    \Exs_k\brackets{\vecnorm{\frac{1}{\nworkers}\sum_{i=1}^\nworkers \sg_i(\x_i^{(k)})}^2}
    =& \Exs_k\brackets{\frac{1}{\nworkers}\sum_{i=1}^\nworkers \brackets{\sg_i(\x_i^{(k)}) - \tg_i(\x_i^{(k)}) + \tg_i(\x_i^{(k)})}}^2 \\
    =& \frac{1}{\nworkers^2}\sum_{i=1}^\nworkers \Exs_k\brackets{\vecnorm{\sg_i(\x_i^{(k)}) - \tg_i(\x_i^{(k)})}^2} + \vecnorm{\frac{1}{\nworkers}\sum_{i=1}^\nworkers\tg_i(\x_i^{(k)})}^2 \\
    \leq& \frac{\vbnd}{\nworkers} + \vecnorm{\frac{\Tg^{(k)}\one}{\nworkers}}^2
    \label{eqn:v_bnd}
\end{align}
where the last inequality is according to the bounded variance assumption. Then, combining \Cref{eqn:init_part1,eqn:v_bnd} and taking the total expectation over all random variables, one can obtain:
\begin{align}
     \Exs \brackets{F(\avgx^{(k+1)}) - F(\avgx^{(k)})}
    \leq&   -\frac{\lr}{2}\Exs\brackets{\vecnorm{\tg(\avgx^{(k)})}^2} - \frac{\lr}{2}(1-\lr\lip)\Exs\brackets{\vecnorm{\frac{\Tg^{(k)}\one}{\nworkers}}^2} + \nonumber \\
        &\frac{\lr\lip^2}{2\nworkers}\Exs\brackets{\fronorm{\X^{(k)}(\I-\J)}^2} + \frac{\lr^2\lip\vbnd}{2\nworkers}.
\end{align}
Summing over all iterates and taking the average,
\begin{align}
    \frac{\Exs\brackets{F(\avgx^{K}) - F(\avgx^{(1)})}}{K}
    \leq& -\frac{\lr}{2}\frac{1}{K}\sum_{i=1}^K\Exs\brackets{\vecnorm{\tg(\avgx^{(k)})}^2} -\frac{\lr}{2}(1-\lr\lip)\frac{1}{K}\sum_{k=1}^K\Exs\brackets{\vecnorm{\frac{\Tg^{(k)}\one}{\nworkers}}^2} + \nonumber \\
        & \frac{\lr\lip^2}{2\nworkers K}\sum_{i=1}^K\Exs\brackets{\fronorm{\X^{(k)}(\I-\J)}^2} + \frac{\lr^2\lip\vbnd}{2\nworkers}.
\end{align}
By minor rearranging, we get
\begin{align}
    \frac{1}{K}\sum_{i=1}^K\Exs\brackets{\vecnorm{\tg(\avgx^{(k)})}^2}
    \leq& \frac{2\Exs\brackets{F(\avgx^{(1)})-F(\avgx^{(K)})}}{\lr K} - \frac{1-\lr \lip}{\nworkers}\frac{1}{K}\sum_{k=1}^K\Exs\brackets{\vecnorm{\frac{\Tg^{(k)}\one}{\nworkers}}^2} + \nonumber \\
        & \frac{\lip^2}{\nworkers K}\sum_{i=1}^K\Exs\brackets{\fronorm{\X^{(k)}(\I-\J)}^2} + \frac{\lr\lip\vbnd}{\nworkers} \\
    \leq& \frac{2\brackets{F(\avgx^{(1)})-F_{\text{inf}}}}{\lr K} - \frac{1-\lr \lip}{\nworkers}\frac{1}{K}\sum_{k=1}^K\Exs\brackets{\vecnorm{\frac{\Tg^{(k)}\one}{\nworkers}}^2} + \nonumber \\
        & \frac{\lip^2}{\nworkers K}\sum_{i=1}^K\Exs\brackets{\fronorm{\X^{(k)}(\I-\J)}^2} + \frac{\lr\lip\vbnd}{\nworkers}. \label{eqn:res1}
\end{align}

Now we complete the first part of the proof. Then, we're going to show that the discrepancies among local models $\Exs\brackets{\fronorm{\X^{(k)}(\I-\J)}^2}$ is upper bounded. According to the update rule of decentralized SGD and the special property of gossip matrix $\mixmat^{(k)} \J = \J \mixmat^{(k)} = \J$, we have
\begin{align}
    \X^{(k)}(\I - \J) 
    =& \parenth{\X^{(k-1)} - \lr \Sg^{(k-1)}}\mixmat^{(k-1)}(\I-\J) \\
    =& \X^{(k-1)}(\I-\J)\mixmat^{(k-1)} - \lr \Sg^{(k-1)}\mixmat^{(k-1)}(\I-\J) \\
    &\vdots \\
    =& \X^{(1)}(\I-\J)\prod_{q=1}^{k-1}\mixmat^{(q)} - \lr \sum_{q=1}^{k-1}\Sg^{(q)}\parenth{\prod_{l=q}^{k-1}\mixmat^{(l)}-\J}.
\end{align}
Since all local models are initiated at the same point, $\X^{(1)}(\I-\J) = 0$. Thus, we can obtain
\begin{align}
    \fronorm{\X^{(k)}(\I-\J)}^2
    =& \lr^2 \fronorm{\sum_{q=1}^{k-1}\Sg^{(q)}\parenth{\prod_{l=q}^{k-1}\mixmat^{(l)}-\J}}^2 \\
    =& \lr^2 \fronorm{\sum_{q=1}^{k-1}\parenth{\Sg^{(q)}-\Tg^{(q)} + \Tg^{(q)}}\parenth{\prod_{l=q}^{k-1}\mixmat^{(l)}-\J}}^2 \\
    \leq& 2 \lr^2 \underbrace{\fronorm{\sum_{q=1}^{k-1}\parenth{\Sg^{(q)}-\Tg^{(q)}}\parenth{\prod_{l=q}^{k-1}\mixmat^{(l)}-\J}}^2}_{T_1} + 2\lr^2\underbrace{  \fronorm{\sum_{q=1}^{k-1} \Tg^{(q)}\parenth{\prod_{l=q}^{k-1}\mixmat^{(l)}-\J}}^2}_{T_2}.
    \label{eqn:bnd_decomp}
\end{align}
For the first term $T_1$ in \eqref{eqn:bnd_decomp}, we have
\begin{align}
    \Exs\brackets{T_1}
    =& \sum_{q=1}^{k-1}\Exs\brackets{\fronorm{\parenth{\Sg^{(q)}-\Tg^{(q)}}\parenth{\prod_{l=q}^{k-1}\mixmat^{(l)}-\J}}^2} \\
    \leq& \sum_{q=1}^{k-1} \rho^{k-q}\Exs\brackets{\fronorm{\Sg^{(q)}-\Tg^{(q)}}^2} \label{eqn:t1_bnd1}\\
    \leq& \nworkers \vbnd \rho \parenth{1 + \rho + \rho^2 + \cdots + \rho^{k-2}} \\
    \leq& \frac{\nworkers\vbnd\rho}{1-\rho} \label{eqn:t1_res}
\end{align}
where \eqref{eqn:t1_bnd1} comes from \Cref{lem:contract}. For the second term $T_2$ in \eqref{eqn:bnd_decomp}, define $\mathbf{A}_{q,p} = \prod_{l=q}^{p}\mixmat^{(l)}-\J$. Then,
\begin{align}
    \Exs\brackets{T_2}
    =& \sum_{q=1}^{k-1}\Exs\brackets{\fronorm{\Tg^{(q)}\mathbf{A}_{q,k-1}}^2} + \sum_{q=1}^{k-1}\sum_{p=1,p\neq q}^{k-1}\Exs\brackets{\trace\{\mathbf{A}_{q,k-1}\tp \Tg^{(q)\top}\Tg^{(p)}\mathbf{A}_{p,k-1}\}} \\
    \leq& \sum_{q=1}^{k-1}\rho^{k-q}\Exs\brackets{\fronorm{\Tg^{(q)}}^2} + \sum_{q=1}^{k-1}\sum_{p=1,p\neq q}^{k-1}\Exs\brackets{\fronorm{\Tg^{(q)}\mathbf{A}_{q,k-1}}\fronorm{\Tg^{(p)}\mathbf{A}_{p,k-1}}} \\
    \leq& \sum_{q=1}^{k-1}\rho^{k-q}\Exs\brackets{\fronorm{\Tg^{(q)}}^2} + \sum_{q=1}^{k-1}\sum_{p=1,p\neq q}^{k-1}\Exs\brackets{\frac{1}{2\epsilon}\fronorm{\Tg^{(q)}\mathbf{A}_{q,k-1}}^2+\frac{\epsilon}{2}\fronorm{\Tg^{(p)}\mathbf{A}_{p,k-1}}^2} \label{eqn:t2_bnd1}\\
    \leq& \sum_{q=1}^{k-1}\rho^{k-q}\Exs\brackets{\fronorm{\Tg^{(q)}}^2} + \sum_{q=1}^{k-1}\sum_{p=1,p\neq q}^{k-1}\Exs\brackets{\frac{\rho^{k-q}}{2\epsilon}\fronorm{\Tg^{(q)}}^2+\frac{\rho^{k-p}\epsilon}{2}\fronorm{\Tg^{(p)}}^2} \label{eqn:t2_bnd2}
\end{align}
where \eqref{eqn:t2_bnd1} follows Young's Inequality: $2ab \leq a^2/\epsilon + \epsilon b^2, \forall \epsilon > 0$ and \eqref{eqn:t2_bnd2} follows \Cref{lem:contract}. Set $\epsilon = \rho^{\frac{p-q}{2}}$, then we have
\begin{align}
    \Exs\brackets{T_2}
    \leq& \sum_{q=1}^{k-1}\rho^{k-q}\Exs\brackets{\fronorm{\Tg^{(q)}}^2} +\frac{1}{2} \sum_{q=1}^{k-1}\sum_{p=1,p\neq q}^{k-1}\sqrt{\rho}^{2k-p-q}\cdot\Exs\brackets{\fronorm{\Tg^{(q)}}^2+\fronorm{\Tg^{(p)}}^2} \\
    =& \sum_{q=1}^{k-1}\rho^{k-q}\Exs\brackets{\fronorm{\Tg^{(q)}}^2} + \sum_{q=1}^{k-1}\brackets{\sqrt{\rho}^{k-q}\Exs\brackets{\fronorm{\Tg^{(q)}}^2} \cdot \sum_{p=1,p\neq q}^{k-1}\sqrt{\rho}^{k-p}} \\
    =& \sum_{q=1}^{k-1}\rho^{k-q}\Exs\brackets{\fronorm{\Tg^{(q)}}^2} + \sum_{q=1}^{k-1}\brackets{\sqrt{\rho}^{k-q}\Exs\brackets{\fronorm{\Tg^{(q)}}^2} \cdot\parenth{ \sum_{p=1}^{k-1}\sqrt{\rho}^{k-p}-\sqrt{\rho}^{k-q}}} \\
    \leq& \frac{\sqrt{\rho}}{1-\sqrt{\rho}} \sum_{q=1}^{k-1}\sqrt{\rho}^{k-q}\Exs\brackets{\fronorm{\Tg^{(q)}}^2}. \label{eqn:t2_res}
\end{align}
Combining \Cref{eqn:t1_res,eqn:t2_res} together,
\begin{align}
    \frac{1}{\nworkers K}\sum_{i=1}^K\Exs\brackets{\fronorm{\X^{(k)}(\I-\J)}^2}
    \leq& \frac{2\lr^2\vbnd \rho}{1-\rho} + \frac{2\lr^2}{\nworkers}\frac{\sqrt{\rho}}{1-\sqrt{\rho}}\frac{1}{K}\sum_{k=1}^K \sum_{q=1}^{k-1}\sqrt{\rho}^{k-q}\Exs\brackets{\fronorm{\Tg^{(q)}}^2} \\
    =& \frac{2\lr^2\vbnd \rho}{1-\rho} + \frac{2\lr^2}{\nworkers}\frac{\sqrt{\rho}}{1-\sqrt{\rho}}\frac{1}{K}\sum_{k=1}^K\brackets{ \Exs\brackets{\fronorm{\Tg^{(k)}}^2}\sum_{q=1}^{K-k}\sqrt{\rho}^{q}} \\
    \leq& \frac{2\lr^2\vbnd \rho}{1-\rho} + \frac{2\lr^2}{\nworkers}\frac{\sqrt{\rho}}{1-\sqrt{\rho}}\frac{1}{K}\sum_{k=1}^K\brackets{ \Exs\brackets{\fronorm{\Tg^{(k)}}^2}\frac{\sqrt{\rho}}{1-\sqrt{\rho}}} \\
    =& \frac{2\lr^2\vbnd \rho}{1-\rho} + \frac{2\lr^2}{\nworkers}\frac{\rho}{(1-\sqrt{\rho})^2}\frac{1}{K}\sum_{k=1}^K\Exs\brackets{\fronorm{\Tg^{(k)}}^2} \label{eqn:discrep}
\end{align}

Note that 
\begin{align}
    \fronorm{\Tg^{(k)}}^2
    =& \sum_{i=1}^\nworkers \vecnorm{\tg_i(\x_i^{(k)})}^2 \\
    =& \sum_{i=1}^\nworkers \vecnorm{\tg_i(\x_i^{(k)}) - \tg(\x_i^{(k)})+ \tg(\x_i^{(k)}) - \tg(\avgx^{(k)}) + \tg(\avgx^{(k)})}^2 \\
    \leq& 3\sum_{i=1}^\nworkers\brackets{\vecnorm{\tg_i(\x_i^{(k)}) - \tg(\x_i^{(k)})}^2 + \vecnorm{\tg(\x_i^{(k)}) - \tg(\avgx^{(k)})}^2 + \vecnorm{\tg(\avgx^{(k)})}^2} \\
    \leq& 3\nworkers\zeta^2 + 3\lip^2\sum_{i=1}^\nworkers\vecnorm{\x_i^{(k)} - \avgx^{(k)}}^2 + 3\nworkers\vecnorm{\tg(\avgx^{(k)})}^2 \\
    =& 3\nworkers\zeta^2 + 3\lip^2\fronorm{\X^{(k)}(\I-\J)}^2 + 3\nworkers\vecnorm{\tg(\avgx^{(k)})}^2. \label{eqn:cor2_bnd2}
\end{align}
Plugging \Cref{eqn:cor2_bnd2} back into \Cref{eqn:discrep}, we have
\begin{align}
    \frac{1}{\nworkers K}\sum_{i=1}^K\Exs\brackets{\fronorm{\X^{(k)}(\I-\J)}^2}
    \leq& \frac{2\lr^2\vbnd \rho}{1-\rho} + \frac{6\lr^2\zeta^2\rho}{(1-\sqrt{\rho})^2} + \frac{6\lr^2\lip^2\rho}{(1-\sqrt{\rho})^2}\frac{1}{\nworkers K}\sum_{i=1}^K\Exs\brackets{\fronorm{\X^{(k)}(\I-\J)}^2} +  \nonumber \\
        & \frac{6\lr^2\rho}{(1-\sqrt{\rho})^2}\frac{1}{K}\sum_{i=1}^K\Exs\brackets{\vecnorm{\tg(\avgx^{(k)})}^2}.
\end{align}
After minor rearranging, we get
\begin{align}
    \frac{1}{\nworkers K}\sum_{i=1}^K\Exs\brackets{\fronorm{\X^{(k)}(\I-\J)}^2}
    \leq& \frac{1}{1-D}\brackets{\frac{2\lr^2\vbnd \rho}{1-\rho} + \frac{6\lr^2\zeta^2\rho}{(1-\sqrt{\rho})^2} +\frac{6\lr^2\rho}{(1-\sqrt{\rho})^2}\frac{1}{K}\sum_{i=1}^K\Exs\brackets{\vecnorm{\tg(\avgx^{(k)})}^2}} \label{eqn:cor2_bnd3}
\end{align}
where $D = \frac{6\lr^2\lip^2\rho}{(1-\sqrt{\rho})^2}$. Then, plugging \Cref{eqn:cor2_bnd3} back into \Cref{eqn:res1}, we have
\begin{align}
    \frac{1}{K}\sum_{i=1}^K\Exs\brackets{\vecnorm{\tg(\avgx^{(k)})}^2}
    \leq& \frac{2[F(\avgx^{(1)})-F_{\text{inf}}]}{\lr K} + \frac{\lr\lip\vbnd}{\nworkers} + \frac{1}{1-D}\frac{2\lr^2\lip^2\vbnd\rho}{1-\rho} + \frac{D\zeta^2}{1-D} + \nonumber \\
        &\frac{D}{1-D} \frac{1}{K}\sum_{i=1}^K\Exs\brackets{\vecnorm{\tg(\avgx^{(k)})}^2}.
\end{align}
It follows that
\begin{align}
    \frac{1}{K}\sum_{i=1}^K\Exs\brackets{\vecnorm{\tg(\avgx^{(k)})}^2}
    \leq& \parenth{\frac{2[F(\avgx^{(1)})-F_{\text{inf}}]}{\lr K} + \frac{\lr\lip\vbnd}{\nworkers}}\frac{1-D}{1-2D} + \parenth{\frac{2\lr^2\lip^2\vbnd\rho}{1-\rho} + \frac{6\lr^2\lip^2\zeta^2\rho}{(1-\sqrt{\rho})^2}}\frac{1}{1-2D} \\
    =& \parenth{\frac{2[F(\avgx^{(1)})-F_{\text{inf}}]}{\lr K} + \frac{\lr\lip\vbnd}{\nworkers}}\frac{1-D}{1-2D} + \frac{2\lr^2\lip^2\rho}{1-\sqrt{\rho}}\parenth{\frac{\vbnd}{1+\sqrt{\rho}} + \frac{3\zeta^2}{1-\sqrt{\rho}}}\frac{1}{1-2D} \\
    \leq& \parenth{\frac{2[F(\avgx^{(1)})-F_{\text{inf}}]}{\lr K} + \frac{\lr\lip\vbnd}{\nworkers}}\frac{1}{1-2D} + \frac{2\lr^2\lip^2\rho}{1-\sqrt{\rho}}\parenth{\frac{\vbnd}{1+\sqrt{\rho}} + \frac{3\zeta^2}{1-\sqrt{\rho}}}\frac{1}{1-2D} \label{eqn:final}
\end{align}
Recall that we require that $\lr\lip \leq (1-\sqrt{\rho})/4\sqrt{\rho}$. Therefore,
\begin{align}
    D = \frac{6\lr^2\lip^2\rho}{(1-\sqrt{\rho})^2}
    \leq \frac{3}{8} < \frac{1}{2} \ \Rightarrow \frac{1}{1-2D} \leq 4.
\end{align}
Plugging the upper bound of $D$ into \eqref{eqn:final} and set $\lr = \sqrt{\frac{m}{K}}$, we have
\begin{align}
    \frac{1}{K}\sum_{i=1}^K\Exs\brackets{\vecnorm{\tg(\avgx^{(k)})}^2}
    \leq& \frac{8[F(\avgx^{(1)})-F_{\text{inf}}] + 4\lip\vbnd}{\sqrt{\nworkers K}} +  \frac{8\nworkers}{K}\frac{\lip^2\rho}{1-\sqrt{\rho}}\parenth{\frac{\vbnd}{1+\sqrt{\rho}} + \frac{3\zeta^2}{1-\sqrt{\rho}}} \\
    =& \mathcal{O}\parenth{\frac{1}{\sqrt{\nworkers K}}} + \mathcal{O}\parenth{\frac{\nworkers}{K}}.
\end{align}
Here we complete the proof.


\section{Spectral Graph Theory}
\noindent The inter-agent communication network is a simple\footnote{A graph is said to be simple if it is devoid of self loops and multiple edges.} undirected graph $\mathcal{G}=(\mathcal{V}, \mathcal{E})$, where $V$ denotes the set of agents or vertices with cardinality $|\mathcal{V}|=m$, and $\mathcal{E}$ the set of edges. If there exists an edge between agents $i$ and $j$, then $(i,j)\in E$. A path between agents $i$ and $j$ of length $n$ is a sequence ($i=p_{0},p_{1},\cdots,p_{n}=j)$ of vertices, such that $(p_{t}, p_{t+1})\in \mathcal{E}$, $0\le t \le n-1$. A graph is connected if there exists a path between all possible agent pairs.
The neighborhood of an agent $n$ is given by $\Omega_{n}=\{j \in \mathcal{V}|(n,j) \in \mathcal{E}\}$.
\noindent The degree of agent $n$ is given by $d_{n}=|\Omega_{n}|$. The structure of the graph is represented by the symmetric $m\times m$ adjacency matrix $\mathbf{A}=[A_{ij}]$, where $A_{ij}=1$ if $(i,j) \in E$, and $0$ otherwise. The degree matrix is given by the diagonal matrix $\mathbf{D}=diag(d_{1}\cdots d_{m})$. The graph Laplacian matrix is defined as $\mathbf{L}=\mathbf{D}-\mathbf{A}$.
\noindent The Laplacian is a positive semidefinite matrix, hence its eigenvalues can be ordered and represented as $0=\lambda_{1}(\mathbf{L})\le\lambda_{2}(\mathbf{L})\le \cdots \lambda_{m}(\mathbf{L})$.
\noindent Furthermore, a graph is connected if and only if $\lambda_{2}(\mathbf{L})>0$